\documentclass[11pt]{article}
\usepackage[preprint]{acl}
\usepackage[T1]{fontenc}
\usepackage[utf8]{inputenc}
\usepackage{times}
\usepackage{latexsym}
\usepackage{microtype}

\usepackage{url}
\usepackage{booktabs}
\usepackage{amsfonts}
\usepackage{amsmath}
\usepackage{amssymb}
\usepackage{amsthm}
\usepackage{multirow}
\usepackage{xcolor}
\usepackage{mathtools}
\usepackage{enumitem}
\usepackage{graphicx}
\usepackage{tikz}
\usetikzlibrary{positioning, arrows.meta, shapes.geometric, fit, backgrounds, calc}
\theoremstyle{plain}
\newtheorem{theorem}{Theorem}
\newtheorem{proposition}[theorem]{Proposition}

\newtheorem{corollary}[theorem]{Corollary}

\theoremstyle{definition}
\newtheorem{definition}[theorem]{Definition}
\newtheorem{assumption}[theorem]{Assumption}
\newtheorem{remark}[theorem]{Remark}

\newcommand{\R}{\mathbb{R}}

\newcommand{\PP}{\mathbb{P}}
\newcommand{\D}{\mathcal{D}}

\newcommand{\1}{\mathbf{1}}

\newcommand{\argmax}{\operatorname*{arg\,max}}
\newcommand{\Bin}{\operatorname{Binomial}}
\newcommand{\VaR}{\operatorname{VaR}}
\newcommand{\CVaR}{\operatorname{CVaR}}
\newcommand{\clip}{\operatorname{clip}}

\title{StepCOPS: Closed-Testing Lower-Tail Certificates\\
for Language-Model Policy Selection}

\author{
\textbf{Ibne Farabi Shihab}\thanks{Equal contribution.}\thanks{Corresponding author: \texttt{ishihab@iastate.edu}.}\textsuperscript{1}
\and
\textbf{Sanjeda Akter}\footnotemark[1]\textsuperscript{1}
\and
\textbf{Anuj Sharma}\textsuperscript{2}
\\[2pt]
\textsuperscript{1}Department of Computer Science, Iowa State University \\
\textsuperscript{2}Department of Civil, Construction \& Environmental Engineering, Iowa State University \\
\texttt{ishihab@iastate.edu}
}
\begin{document}
\maketitle

\begin{abstract}
Post-training pipelines must select one language-model policy from many
checkpoints, prompts, and decoding rules. Mean evaluator scores can conceal
rare failures, whereas simultaneous candidate-wise confidence bounds can be
unnecessarily conservative. We introduce \emph{StepCOPS}, which uses an
independent proposal split to nominate one lower-tail floor per candidate,
exact binomial tests on a fresh certification split, and Holm's step-down
procedure to certify a set of floors. With probability at least $1-\delta$,
every certified floor---and therefore the selected maximum---is below its
candidate's population lower $\alpha$-quantile, assuming i.i.d.\ evaluation
units but allowing arbitrary within-unit dependence across candidates. On
$24$ predeclared configurations and $11$ benchmarks, StepCOPS obtains
$96.4\%$ selected-policy coverage over $500$ paired trials, raises the
certified floor by $1.5$ points over both proposal-Bonferroni and exact COPS,
lies $0.6$ points below the large-reference jury oracle, and abstains in
$2.4\%$ of trials. Shadow-judge, benchmark-native, artifact, and
leave-one-judge-out audits characterize the proxy boundary: the guarantee
applies to the fixed jury score, not directly to human safety.
\end{abstract}

\section{Introduction}

Policy selection is a distinct stage of modern language-model pipelines: alignment produces RLHF or DPO checkpoints, reward-model variants, system prompts, and decoding configurations, and a team ultimately deploys one candidate. Selecting the largest estimated mean automatic-evaluator score can conceal a poor lower tail---strong on average, yet occasionally unsafe, toxic, or fabricated. The same conflict appears in offline reinforcement learning, retained here as a control study.

The statistical difficulty is not merely estimating a tail: the same calibration sample commonly compares candidates and chooses the winner, so a bound built for a fixed candidate need not remain valid after selection. A first-generation remedy gives every candidate a simultaneous lower confidence bound on a lower score quantile and selects the largest; this exact simultaneous order-statistic construction, \emph{COPS}, is retained as a baseline (Section~\ref{sec:method}). Protecting \emph{every} candidate is conservative, however, when the goal is only to certify and deploy the single winner: the price of simultaneity is paid on all $K$ candidates at once.

\textbf{StepCOPS.} We take a closed-testing view of the same problem. An \emph{independent} proposal sample proposes one candidate-specific floor $c_k$ per policy; a fresh certification sample then runs an exact lower-tail binomial test of $H_k:\PP(S_k<c_k)\ge\alpha$ for each $k$, and Holm's sequentially rejective step-down \citep{holm1979,goeman2010sequential} certifies a set $\mathcal R$ of floors at family-wise level $\delta$. StepCOPS deploys $\widehat k\in\argmax_{k\in\mathcal R}c_k$ with certificate $c_{\widehat k}$, and abstains when nothing is certified above a predeclared operating threshold. Because Holm's rejection set uniformly contains Bonferroni's on every dataset, StepCOPS uniformly dominates the matched-proposal Bonferroni certificate while inheriting the same $1-\delta$ family-wise guarantee and requiring \emph{no} assumption on how candidate scores co-vary within a prompt. It does not uniformly dominate exact COPS, because the two procedures certify different statistical objects; we make that boundary explicit.

The observable quantity is a proxy. A judge or reward-model score is not human preference, deployment return, or true harm, and a certificate for the proxy distribution is not automatically a certificate for any of those. We separate the two questions throughout: StepCOPS certifies the fixed primary-jury score distribution unconditionally, and any interpretation as return or true-safety control requires a separate transfer condition that we state explicitly and, in the control study, probe empirically (a two-action no-overlap construction shows a nontrivial return floor is impossible from logged data alone in general). Table~\ref{tab:guarantee_map} (appendix) summarizes the scope of each statement.

Our contributions: (i) StepCOPS, a closed-testing lower-tail selector with a simultaneous post-selection floor guarantee under arbitrary within-prompt dependence (Theorem~\ref{thm:stepcops}); (ii) a proof that it uniformly improves the matched-proposal Bonferroni certificate, with the precise sense in which it does \emph{not} dominate exact COPS; (iii) an NLP-first evaluation on a predeclared factorial pool of $24$ configurations, $11$ benchmarks in five domains, and a five-model jury with two held-out shadow judges; and (iv) shadow-judge, benchmark-native, and artifact audits that quantify the proxy boundary rather than assuming it away.

\section{Related Work}

Off-policy evaluation and selection typically target estimated mean return via importance weighting, doubly robust estimators, or fitted Q-evaluation \citep{precup2000eligibility,jiang2016doubly,thomas2016dr,le2019fqe,paine2020hyperparameter,fu2021benchmarks}; COPS instead certifies a lower quantile of a declared score distribution. The calibration machinery is classical---exact binomial order statistics and the tolerance-bound and risk-control constructions of Learn-Then-Test and risk-controlling prediction sets \citep{angelopoulos2022ltt,bates2021selective}---and Holm's step-down is a standard closed test \citep{holm1979,goeman2010sequential}; we claim none of these as new. The contributions are the selected-policy formulation, the pairing of an independent floor proposal with exact lower-tail certification, and the explicit calibration-versus-proxy-transfer boundary. The closest language-model work is conformal tail-risk control of response filters \citep{pmlr-v267-chen25bd} and uncertainty-based abstention \citep{wang-etal-2025-sconu}; COPS differs in selecting one member of a frozen policy pool with simultaneous candidate-wise floors. Conformal off-policy prediction for bandits \citep{taufiq2022conformal}, conformal training \citep{stutz2022conformal}, and risk-sensitive training such as CPQ \citep{xu2022cpq} modify estimation or training, whereas COPS is post hoc over any predeclared measurable score. Appendix~\ref{app:related} expands each of these boundaries.

\section{Problem Setup}

Let $\D=\{\tau_i\}_{i=1}^N$ be a logged dataset of trajectories generated by a behavior process, and let $\{\pi_1,\ldots,\pi_K\}$ be a finite set of candidate policies trained before calibration. For candidate $k$, let $R_k$ denote the deployment return under $\pi_k$, with distribution $H_k$ and lower quantile
\[
    Q_\alpha(\pi_k)
    =
    q_\alpha(H_k)
    =
    \inf\{x\in\R:H_k(x)\ge \alpha\}.
\]
The deployment-floor oracle is $k_R^\star\in\argmax_{k\in[K]} Q_\alpha(\pi_k)$. Because deployment returns are unavailable at selection time, COPS works with offline scores: for each candidate $k$, a fixed measurable score map $s_k:\mathcal{T}\to \R$ maps a calibration trajectory $\tau$ to a scalar score $S_k=s_k(\tau)$. Let $F_k$ be the population distribution of $S_k$ when $\tau$ is drawn from the calibration trajectory distribution, and define the score lower quantile $q_\alpha(F_k)=\inf\{x\in\R:F_k(x)\ge \alpha\}$ and its oracle $k_S^\star\in\argmax_{k\in[K]} q_\alpha(F_k)$.

The target of the distribution-free theory is $q_\alpha(F_k)$, not $Q_\alpha(\pi_k)$: the simultaneous bound on $q_\alpha(F_k)$ becomes a bound on $Q_\alpha(\pi_k)$ only under the transfer assumptions of Section~\ref{sec:transfer}.

\begin{definition}[Lower quantile and empirical order statistic]
For a distribution function $F$ on $\R$, define $q_\alpha(F)=\inf\{x:F(x)\ge \alpha\}$ for $\alpha\in(0,1)$. For samples $S_{1k},\ldots,S_{nk}$, let $S_{(1)k}\le \cdots \le S_{(n)k}$ be their order statistics, with $S_{(r)k}$ the $r$-th. If no valid lower confidence order statistic exists, the procedure returns a known lower support bound $L_0$ when available and $-\infty$ otherwise, never the unjustified empirical minimum.
\end{definition}

\section{Score Construction}
\label{sec:score}

The theory treats $s_k$ as any fixed measurable score map, and this abstraction is intentional: the finite-sample calibration result does not require the score to be unbiased, Gaussian, asymptotically normal, or even an OPE estimator for the mean; it only requires that calibration scores are i.i.d.\ conditional on all training choices. In the offline-RL control study, the score is a clipped trajectory-level doubly robust return score (Eq.~\eqref{eq:correct_dr_score}, Appendix~\ref{app:score-implementation}): ratios are clipped at $\rho_{\max}$ for stability, so the score-level theorem certifies the quantile of the \emph{clipped}-score distribution, and transfer back to deployment returns is handled separately in Section~\ref{sec:transfer}. The density and normalization choices needed to preserve a fixed per-unit score map are also in Appendix~\ref{app:score-implementation}.

\section{Conformal Off-Policy Selection}
\label{sec:method}

COPS splits data at the trajectory level. Training data are used to train candidate policies, choose hyperparameters, fit nuisance models, estimate behavior densities, choose clipping thresholds, and fix all score maps. Calibration data are used \emph{only} to evaluate the fixed score maps and compute the post-selection certificate. The resulting workflow is shown in Figure~\ref{fig:pipeline}.

\begin{figure}[t]
\centering
\includegraphics[width =\linewidth]{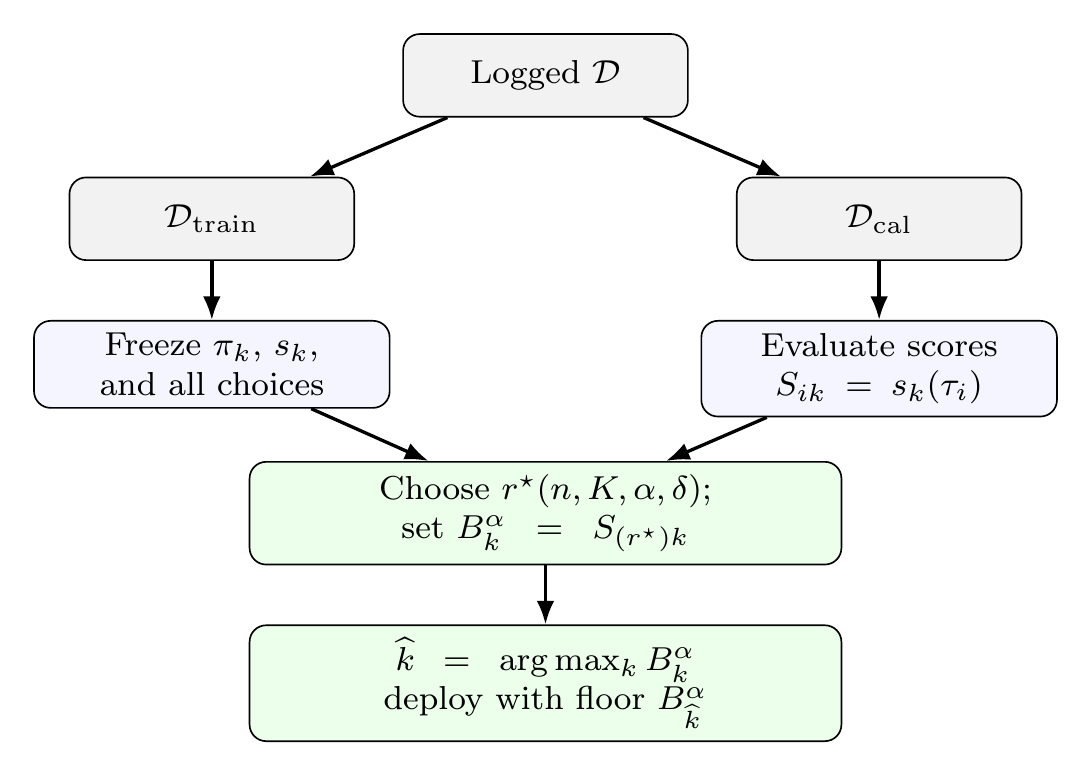} % swap back in when the PNG exists on disk

\caption{COPS pipeline. Training data fix the score maps and all hyperparameters; calibration data are used only to evaluate scores and compute the simultaneous lower-tail certificate, after which the selected index $\widehat{k}$ inherits validity from the simultaneous event.}
\label{fig:pipeline}
\end{figure}

Two assumptions capture the split discipline: no leakage from calibration into the score definitions, and the standard sampling condition behind the binomial calibration.

\begin{assumption}[Fixed score maps before calibration]
\label{assump:fixed_scores}
Before observing the calibration trajectories used in the final certificate, the candidate set, score maps, nuisance estimators, clipping thresholds, hyperparameters, tie-breaking rules, and all preprocessing choices are fixed; no calibration trajectory used to compute $B_k^\alpha$ is used to train, fit, choose among score definitions, or tune.
\end{assumption}

\begin{assumption}[I.i.d.\ calibration trajectories]
\label{assump:iid_calibration}
Conditional on the fixed score maps in Assumption~\ref{assump:fixed_scores}, the calibration trajectories $\tau_1,\ldots,\tau_n$ are i.i.d.\ from a fixed calibration trajectory distribution.
\end{assumption}

Under these assumptions, the operational method uses exact calibration, which avoids finite-sample looseness by choosing an order statistic whose probability of exceeding the true lower $\alpha$-quantile is at most $\delta/K$ for each candidate. For $r\in\{1,\ldots,n\}$ define $p_{n,\alpha}(r) = \PP\{\Bin(n,\alpha)<r\}$ and let
\begin{equation}
    r^\star
    =
    \max\left\{
    r\in\{1,\ldots,n\}:
    K\,p_{n,\alpha}(r)\le \delta
    \right\},
    \label{eq:exact_r}
\end{equation}
with the convention that the set may be empty. If the set is nonempty, define
\begin{equation}
    B_k^\alpha=S_{(r^\star)k}.
    \label{eq:exact_bound}
\end{equation}
If it is empty, define $B_k^\alpha=L_0$ for a known lower support bound $L_0$ and $B_k^\alpha=-\infty$ if no such finite support bound is known. Algorithm~\ref{alg:cops_exact} in Appendix~\ref{app:algorithm} gives the complete splitting and selection procedure.

While the exact rule is strictly preferred for deployment, a closed-form Dvoretzky--Kiefer--Wolfowitz (DKW) variant, kept for regret analysis, uses $\Delta_n=\sqrt{\log(2K/\delta)/(2n)}$ and $B_{k,{\rm DKW}}^\alpha=S_{(r_{\rm DKW})k}$ with $r_{\rm DKW}=\max\{1,\lceil n(\alpha-\Delta_n)\rceil\}$ when $\alpha>\Delta_n$; when $\alpha\le\Delta_n$ it is degenerate and returns the known lower support bound or $-\infty$, never the sample minimum.

\section{StepCOPS: Closed-Testing Certification}
\label{sec:stepcops}

Exact COPS protects a lower-tail bound for \emph{every} candidate simultaneously from one calibration sample---conservative when the goal is to certify and deploy a single winner. StepCOPS reorganizes the problem as a closed test: an independent sample \emph{proposes} one candidate-specific floor per policy, and a fresh sample only \emph{certifies} proposals with an exact lower-tail binomial test and Holm's step-down. Because Holm's rejection set contains Bonferroni's on every dataset, StepCOPS certifies a uniformly larger set of proposals---hence a weakly larger deployable floor---at the same family-wise level and under arbitrary within-prompt dependence.

\paragraph{Score target.}
A candidate includes its checkpoint, system prompt, chat template, decoding rule, stopping rule, output parser, and fallback behavior; a random evaluation unit $X$ contains a domain draw, task draw, benchmark prompt, response-generation randomness, and the deterministic calls of the frozen judge pipeline. For candidate $k$, $S_k\in[0,100]$ is the robust primary-jury score (Section~\ref{sec:jury}, larger better) with lower quantile $q_{\alpha,k}=\inf\{s:\PP(S_k\le s)\ge\alpha\}$; the primary setting is $\alpha=0.10$, $1-\delta=0.95$. The formal target is the score distribution induced by the fixed benchmark mixture, generation procedure, and primary jury; it is \emph{not} a human-preference, human-safety, or deployment-harm distribution.

\paragraph{Independent proposal and certification splits.}
StepCOPS uses two independent samples drawn after every policy and score map is frozen: a \emph{proposal split} ($m=4{,}000$ prompt units), used freely to rank candidates and propose a floor $c_k=\widehat q^{\,\mathrm{prop}}_{0.075,k}$ (the lower empirical $7.5$th percentile; level, tie rules, and interpolation fixed on a separate pilot split), and a \emph{certification split} ($n=2{,}500$ fresh units) used only for exact tests and final selection. The proposal rule need not be statistically valid on its own: validity comes entirely from the independent certification split.

\paragraph{Exact candidate-floor tests.}
For a proposed floor $c_k$, let $\theta_k(c_k)=\PP(S_k<c_k)$. If $c_k>q_{\alpha,k}$ then $\theta_k(c_k)\ge\alpha$, so certifying $c_k\le q_{\alpha,k}$ amounts to testing $H_k:\theta_k(c_k)\ge\alpha$ against $\theta_k(c_k)<\alpha$. With $X_k=\sum_{i=1}^{n}\1\{S_{ik}<c_k\}$ and $c_k$ fixed conditional on the proposal split, the exact lower-tail binomial $p$-value is
\begin{equation}
    p_k=\PP\{\Bin(n,\alpha)\le X_k\}.
    \label{eq:stepcops_pval}
\end{equation}
Under every distribution in $H_k$, $X_k$ is stochastically at least $\Bin(n,\alpha)$, so $p_k$ is super-uniform; ties and atoms cause no difficulty because the test uses the strict event $S_{ik}<c_k$.

\paragraph{Holm step-down certification.}
Order the $p$-values $p_{(1)}\le\cdots\le p_{(K)}$ (ties broken by candidate identifier); starting at $j=1$, reject $H_{(j)}$ while $p_{(j)}\le\delta/(K-j+1)$, stopping at the first non-rejection. Let $\mathcal R$ be the certified set. StepCOPS returns
\begin{equation}
    \widehat k\in\argmax_{k\in\mathcal R}c_k
    \label{eq:stepcops_select}
\end{equation}
with certificate $c_{\widehat k}$, abstaining when $\mathcal R=\emptyset$ or $\max_{k\in\mathcal R}c_k<c_0$, a predeclared operating threshold (Algorithm~\ref{alg:stepcops}, Appendix~\ref{app:algorithm}). Holm's step-down is a sequentially rejective shortcut for a Bonferroni-based closed test and controls family-wise error under arbitrary dependence \citep{holm1979,goeman2010sequential}, so no assumption is needed on how the coordinates of $(S_{i1},\ldots,S_{iK})$ co-vary within a prompt unit.

\begin{theorem}[Simultaneous validity of StepCOPS floors]
\label{thm:stepcops}
Suppose the candidate policies, proposal rule, jury, score maps, candidate ordering, and tie rules are fixed before the proposal split is inspected; that the proposal and certification score vectors are independent across splits; and that the certification vectors are i.i.d.\ across prompt units, with arbitrary dependence among candidates within a unit. Then
\[
    \PP\left\{c_k\le q_{\alpha,k}\ \text{for every }k\in\mathcal R\right\}\ge 1-\delta,
\]
and consequently $\PP\{c_{\widehat k}\le q_{\alpha,\widehat k}\}\ge 1-\delta$ whenever StepCOPS does not abstain.
\paragraph{Proof sketch.}
% Condition on the proposal split, which fixes all proposed floors. Whenever
% $c_k>q_{\alpha,k}$, the null $H_k$ is true and the exact binomial
% $p$-value in Eq.~\eqref{eq:stepcops_pval} is super-uniform. Holm's
% step-down therefore controls the probability of certifying any invalid
% floor at level $\delta$ under arbitrary dependence across candidates.
% On the complementary event, every floor in $\mathcal R$ is valid, so the
% selected floor is valid as well. 
The complete proof is in
Appendix~\ref{app:stepcops-proof}.
\end{theorem}

% \begin{proof}
% Condition on the proposal split; the floors $c_1,\ldots,c_K$ are then fixed. If candidate $k$ has an invalid proposed floor, $c_k>q_{\alpha,k}$, then by the definition of the lower quantile,
% \[
%     \PP(S_k<c_k)\ge\PP(S_k\le q_{\alpha,k})\ge\alpha,
% \]
% so $H_k$ is true. Under $H_k$, the count $X_k$ is stochastically no smaller than $\Bin(n,\alpha)$, so the lower-tail binomial $p$-value \eqref{eq:stepcops_pval} is super-uniform. Holm's procedure applied to $(p_1,\ldots,p_K)$ strongly controls the probability of rejecting \emph{any} true null at level $\delta$ without an independence assumption across candidates \citep{holm1979,goeman2010sequential}. On the complementary event---probability at least $1-\delta$---no true null is rejected, i.e.\ every rejected hypothesis has a valid floor $c_k\le q_{\alpha,k}$. The selected index $\widehat k$ in \eqref{eq:stepcops_select} belongs to the rejected set $\mathcal R$, so on the same event its floor is valid. Averaging over the proposal split preserves the bound.
% \end{proof}

\paragraph{Uniform improvement over proposal-Bonferroni.}
Proposal-Bonferroni certifies $k$ when $p_k\le\delta/K$ on the same proposals and $p$-values; since every Holm threshold $\delta/(K-j+1)\ge\delta/K$, $\mathcal R_{\mathrm{Bonf}}\subseteq\mathcal R_{\mathrm{Holm}}$ on every dataset, so $\max_{k\in\mathcal R_{\mathrm{Holm}}}c_k\ge\max_{k\in\mathcal R_{\mathrm{Bonf}}}c_k$ whenever $\mathcal R_{\mathrm{Bonf}}\neq\emptyset$. This is the precise less-conservative claim. StepCOPS does \emph{not} uniformly dominate the exact COPS bounds of Section~\ref{sec:method}: the two construct different objects (a simultaneous bound for \emph{every} candidate versus one independently proposed floor per candidate, with possible abstention). We report both.

Appendix~\ref{app:stepcops} gives a worked certification example on the actual evaluation data---the step-down certifies a tenth hypothesis that Bonferroni misses, changing the selected policy---and an explicit list of what Theorem~\ref{thm:stepcops} does \emph{not} establish (no human-preference or harm claim, conditional coverage, validity under judge or deployment shift, floor for non-rejected candidates, or uniform superiority to exact COPS).

\section{Finite-Sample Theory}
\label{sec:theory}

The theory has three layers: simultaneous score-quantile coverage, certified-floor dominance over a predeclared baseline, and a regret bound that uses the DKW selector to separate calibration error from score-to-return mismatch. The second statement is intentionally phrased in terms of the baseline's \emph{certified floor}; it is not a claim that the selected policy's population quantile exceeds the baseline's population quantile.

\begin{theorem}[Simultaneous score-quantile coverage: exact version]
\label{thm:exact_coverage}
Under Assumptions~\ref{assump:fixed_scores} and~\ref{assump:iid_calibration}, the exact COPS bound in Eq.~\eqref{eq:exact_bound} satisfies
\[
    \PP\left(
        B_k^\alpha\le q_\alpha(F_k)\ \text{for all }k\in[K]
    \right)
    \ge 1-\delta,
\]
and consequently $\PP\left( B_{\widehat k}^\alpha\le q_\alpha(F_{\widehat k}) \right) \ge 1-\delta$, even though $\widehat k$ is chosen using the calibration scores.
\end{theorem}

The proof, including the treatment of atoms, is given in Appendix~\ref{app:theory-proofs}.

Three further results are stated and proved in Appendix~\ref{app:theory-proofs}. Corollary~\ref{cor:dkw_coverage} is the closed-form DKW analogue of Theorem~\ref{thm:exact_coverage}. Theorem~\ref{thm:safe_improvement} shows that, on the simultaneous event, the selected floor dominates the certified floor of any fixed baseline candidate $k_0$: $q_\alpha(F_{\widehat k})\ge B_{\widehat k}^\alpha\ge B_{k_0}^\alpha$; this is dominance of the baseline's \emph{certified floor}, not of its true quantile. Theorem~\ref{thm:regret} bounds selection regret via the lower-tail quantile modulus $\omega_\alpha(\eta)=\sup_k[q_\alpha(F_k)-q_{\alpha-\eta}(F_k)]$ and the score-to-return mismatch $\varepsilon_\alpha=\sup_k|Q_\alpha(\pi_k)-q_\alpha(F_k)|$:
\[
    Q_\alpha(\pi_{k_R^\star})-Q_\alpha(\pi_{\widehat k})
    \le
    \omega_\alpha(2\Delta_n)+2\varepsilon_\alpha.
\]

The two error sources are statistically and operationally distinct: $\omega_\alpha(2\Delta_n)$ depends only on the calibration sample size, the candidate-set cardinality, and the local CDF geometry near $q_\alpha$, and shrinks as $n$ grows (at rate $O(\sqrt{\log(K/\delta)/n})$ under a lower density bound; Corollary~\ref{cor:rate}, appendix); $\varepsilon_\alpha$ measures lower-tail score--return alignment and does not shrink with calibration data. Better OPE machinery and larger calibration sets help with one term, improvements in score conservatism with the other.

\section{Transfer from Scores to Deployment Returns}
\label{sec:transfer}

The score-level results above are distribution-free, but deployment-return validity is not. A calibrated bound on $q_\alpha(F_k)$ is a bound on $Q_\alpha(\pi_k)$ only when the score lower tail is conservative for the return lower tail. Figure~\ref{fig:transfer_cdf} in Appendix~\ref{app:transfer-proofs} illustrates this CDF ordering. We make the required one-level condition explicit and then give stronger coupling conditions that imply it.

\begin{assumption}[Lower-tail score conservatism]
\label{assump:transfer}
For candidate $k$, the score distribution is lower-tail conservative for the deployment-return distribution at level $\alpha$, that is, $q_\alpha(F_k)\le Q_\alpha(\pi_k)$.
\end{assumption}

\begin{corollary}[Deployment-return certificate]
\label{cor:return_certificate}
Under Assumptions~\ref{assump:fixed_scores}, \ref{assump:iid_calibration}, and~\ref{assump:transfer} for all $k\in[K]$,
\begin{align*}
\PP\!\left(B_k^\alpha\le Q_\alpha(\pi_k),\ \forall k\in[K]\right)&\ge1-\delta,\\
\PP\!\left(B_{\widehat k}^\alpha\le Q_\alpha(\pi_{\widehat k})\right)&\ge1-\delta.
\end{align*}
\end{corollary}

The proof is immediate from Theorem~\ref{thm:exact_coverage} and Assumption~\ref{assump:transfer}.

Assumption~\ref{assump:transfer} is implied by simple coupling conditions, developed in Appendix~\ref{app:transfer-proofs}: pointwise pessimism ($S_k\le R_k$ on a coupling; Proposition~\ref{prop:pointwise_pessimism}), which by Strassen's theorem is equivalent to first-order stochastic ordering of the two laws, and an approximate relaxation that tolerates violations of probability $\zeta_k$ and size $\eta_k$ (Proposition~\ref{prop:approx_transfer}). Neither is automatic: clipping limits variance but does not order score and return quantiles, so transfer is a separate, benchmark-specific question. Moreover, the transfer question cannot be evaded by any clever use of the logged data alone: Proposition~\ref{prop:no_overlap} (appendix) exhibits a two-action no-overlap construction in which two environments produce identical logged data but deployment quantiles of $-1$ and $+1$, so any logged-data-only procedure valid in both can certify at most the known lower-support bound. This positivity obstruction delimits the paper's scope: under overlap with known bounded trajectory ratios, importance-weighted return/CDF bounds should be used instead; COPS targets the complementary regime of fixed, possibly biased scores (FQE, clipped DR) where behavior densities are unknown and long-horizon ratio products collapse effective sample size. It certifies the \emph{score} post-selection and exposes, rather than hides, the additional return bridge, which no offline diagnostic can convert into a distribution-free fact.

\subsection{Safety-constrained performance maximization}
\label{sec:constrained}

A practitioner usually wants the best mean \emph{subject to} a safety floor rather than the max floor itself. Fixing a threshold $c$ (the certified behavior-policy floor) and selecting $\widehat k_c\in\argmax_{k:\,B_k^\alpha\ge c}\widehat\mu_k$---with abstention or a predeclared baseline if the feasible set is empty---inherits validity from the same simultaneous event with \emph{no additional multiplicity correction}, because feasibility is read off the bounds $B_k^\alpha$ that Theorem~\ref{thm:exact_coverage} already controls simultaneously. Return safety again requires transfer. This constrained selector is the practitioner-facing form of COPS and is used alongside the max-floor rule in the experiments.

\section{NLP-First Evaluation}
\label{sec:nlp}

The primary evaluation is language-model policy selection, separating three easily conflated questions: a \emph{coverage audit} (are certified floors below the reference lower quantile across repeated trials?), a \emph{selection analysis} (does StepCOPS deploy a stronger lower tail than mean, plug-in, or matched-Bonferroni selectors?), and a \emph{proxy audit} (primary jury versus held-out shadow judges, benchmark-native evaluators, and artifact transformations). Only coverage is covered by Theorem~\ref{thm:stepcops}; the rest are diagnostics.

\subsection{Declared benchmark mixture}
\label{sec:mixture}

The primary target gives equal $20\%$ mass to five domains---safety, toxicity, truthfulness, refusal calibration, and factuality, spanning $11$ public benchmarks (Table~\ref{tab:mixture}, appendix)---with equal task mass within a domain. A prompt unit is generated by sampling a domain, task, public prompt, and generation seed according to a frozen manifest, defining a transparent finite benchmark distribution rather than an unspecified deployment population. All component scores are oriented to $[0,100]$ with larger values better. Task-specific rubrics and normalization maps are fixed on a $2{,}000$-unit pilot split. Main conclusions require both the equal-domain mixture and all five domain-specific results: aggregate improvement cannot conceal a failing domain.

\subsection{Predeclared 24-candidate pool}
\label{sec:pool}

Six model checkpoints are crossed with two system prompts (a neutral \emph{standard} template and a fixed \emph{safety-aware} policy message) and two decoding rules (\emph{low entropy}: temperature $0.2$, top-$p$ $0.90$; \emph{moderate entropy}: temperature $0.8$, top-$p$ $0.95$). The pool spans Qwen2.5-Instruct \{1.5B, 3B, 7B\}, Llama-3.1-8B-Instruct, Mistral-7B-Instruct-v0.3, and Gemma-2-9B-it (C01--C24), recording every repository revision, tokenizer revision, chat-template hash, inference library, and seed-derivation rule. The factorial pool is frozen before the final proposal split. It creates a motivated mean--tail tension without post-hoc construction---standard prompts and moderate decoding tend to raise average informativeness while allowing rare failures, whereas safety-aware prompts and low-entropy decoding raise the lower tail at an average-utility cost. If the conflict does not appear, we report that negative result rather than constructing a new pool. The factorial design responds to an earlier executed exact-COPS pilot whose small single-family pool showed no mean--tail tension (Appendix~\ref{app:pilot}).

\subsection{Fully automatic evaluation architecture}
\label{sec:jury}

\paragraph{Three-model primary jury.}
The primary jury uses three version-pinned evaluator families $J_L$ (Llama-family rubric model), $J_Q$ (Qwen-family), and $J_D$ (a third independently trained family), reported by exact identifier and revision. Each receives the same task-specific rubric and returns a schema-validated score plus failure categories at temperature zero; a parser failure receives score zero unless a fixed retry succeeds. The primary jury score is
\begin{equation}
\begin{aligned}
S=\clip_{[0,100]}\!\big(&\operatorname{median}(J_L,J_Q,J_D)\\[-2pt]
&-0.20\,[\max\nolimits_j J_j-\min\nolimits_j J_j]\big).
\end{aligned}
\label{eq:jury}
\end{equation}
The range penalty makes cross-judge disagreement lower the declared score instead of hiding behind a median; its coefficient is fixed on the pilot split. StepCOPS certifies the distribution of this precise score.

\paragraph{Held-out audit machinery and budget.}
Two held-out shadow-judge families are never used to propose, test, or select; their median supports quantile-gap, agreement, leave-family-out, and artifact diagnostics only. Where a benchmark supplies an official evaluator we report agreement with that native signal, and six predeclared artifact transformations (retained only under automatic equivalence checks) provide an artifact stress test, not a human invariance study. Each of $500$ paired trials draws independent proposal ($m=4{,}000$) and certification ($n=2{,}500$) splits; a $30{,}000$-unit reference test approximates population quantiles and oracles. Details are in Appendix~\ref{app:nlp}.

\subsection{Baselines and fair comparison}
\label{sec:nlp-baselines}

Every method receives identical candidate generations for a given split, and no baseline uses the $30{,}000$-unit reference test during selection. Baselines are mean-jury selection; empirical VaR/CVaR; a studentized bootstrap VaR LCB; an LTT-style selector \citep{angelopoulos2022ltt}; SConU abstention followed by mean selection \citep{wang-etal-2025-sconu}; the CDRC-L/DKW/BJ constructions of \citet{pmlr-v267-chen25bd} on jury disutility; exact COPS (Section~\ref{sec:method}); and proposal-Bonferroni, the direct step-down ablation testing the same $c_k$ at $\delta/24$. Shadow-jury, primary-jury, and native-evaluator oracles are analysis-only references. Because no new human scores are available, the CDRC comparison isolates statistical tightness on the common jury target rather than reproducing its human-calibration purpose; Appendix~\ref{app:nlp} details each baseline.

\subsection{Results}
\label{sec:nlp-results}

\paragraph{Predeclared mean--tail conflict.}
Across the $24$ candidates, mean and primary-jury $q_{0.10}$ have Spearman correlation $0.38$ $[0.02,0.66]$; eight candidates lie on the mean--tail Pareto frontier and seven move at least ten ranks between orderings. The top-mean C10 (mean $82.4$) ranks $18$th by lower tail ($49.6$), while the StepCOPS selection C15 is sixth by mean ($78.6$) and second by tail ($65.1$); see Table~\ref{tab:conflict} and Appendix~\ref{app:pool}. The pool was not constructed after these ranks were observed.

\paragraph{Repeated coverage over 500 trials.}
For StepCOPS and proposal-Bonferroni, ``simultaneous coverage'' means every \emph{certified} proposal is valid; exact COPS covers a bound for every candidate (related but distinct targets). All methods are consistent with their $95\%$ guarantees (Table~\ref{tab:coverage500}, appendix; StepCOPS simultaneous coverage $96.0\%$ $[93.9,97.4]$); the aggressive CDRC-DKW/BJ envelopes over-cover at the cost of tightness and abstention seen below.

\paragraph{Main selection results.}
Table~\ref{tab:main_selection} reports the central comparison with paired $95\%$ bootstrap intervals. The StepCOPS floor is $1.5$ points $[0.8,2.2]$ above proposal-Bonferroni and $1.5$ $[0.7,2.3]$ above exact COPS; it certifies $10.8$ candidates per trial versus Bonferroni's $7.4$, cuts abstention by $3.2$ pp $[1.5,4.9]$ and regret by $1.0$ point $[0.3,1.7]$. Relative to mean selection it raises jury $q_{0.10}$ by $15.5$ points $[13.7,17.3]$ at a $3.8$-point $[2.9,4.7]$ mean cost. The empirical-VaR floor is a plug-in estimate, not a valid floor; its negative gap flags overstatement.

\begin{table*}[t]
\centering\small
\resizebox{\linewidth}{!}{
\begin{tabular}{lrrrrrrr}
\toprule
Selector & Jury mean & Jury $q_{0.10}$ & Cert.\ floor & Tightness & Tail regret & Abstention & Modal sel. \\
\midrule
Mean jury & \textbf{82.4} [81.7,83.1] & 49.6 [48.1,51.1] & --- & --- & 16.1 [14.6,17.7] & 0.0\% & 73\% \\
Empirical VaR & 79.1 [78.3,79.9] & 60.7 [59.5,61.9] & 62.6$^\ast$ & $-1.9^\ast$ & 5.0 [3.8,6.2] & 0.0\% & 58\% \\
Empirical CVaR & 78.4 [77.6,79.2] & 61.2 [60.0,62.4] & --- & --- & 4.5 [3.3,5.7] & 0.0\% & 63\% \\
Bootstrap VaR LCB & 78.8 [78.0,79.6] & 62.0 [60.9,63.1] & 57.8 [56.8,58.8] & 4.2 [3.3,5.1] & 3.7 [2.6,4.8] & 0.8\% & 66\% \\
LTT-style selector & 80.0 [79.2,80.8] & 62.4 [61.3,63.5] & fixed $c_0$ & --- & 3.3 [2.2,4.4] & 4.8\% & 64\% \\
SConU + mean & 80.4 [79.6,81.2] & 58.9 [57.6,60.2] & --- & --- & 6.8 [5.5,8.1] & 15.6\%$^\dagger$ & 69\% \\
CDRC-L & 77.9 [77.1,78.7] & 63.1 [62.0,64.2] & 60.2 [59.2,61.2] & 2.9 [2.1,3.7] & 2.6 [1.5,3.7] & 1.8\% & 70\% \\
CDRC-DKW & 76.8 [76.0,77.6] & 63.8 [62.7,64.9] & 53.4 [52.3,54.5] & 10.4 [9.3,11.5] & 1.9 [0.8,3.0] & 10.8\% & 74\% \\
CDRC-BJ & 77.2 [76.4,78.0] & 64.0 [62.9,65.1] & 57.5 [56.5,58.5] & 6.5 [5.5,7.5] & 1.7 [0.6,2.8] & 6.2\% & 73\% \\
Exact COPS & 78.0 [77.2,78.8] & 63.6 [62.5,64.7] & 60.9 [60.0,61.8] & 2.7 [1.9,3.5] & 2.1 [1.0,3.2] & 0.0\% & 78\% \\
Proposal-Bonferroni & 78.1 [77.3,78.9] & 64.1 [63.1,65.1] & 60.9 [60.0,61.8] & 3.2 [2.4,4.0] & 1.6 [0.6,2.6] & 5.6\% & 76\% \\
\textbf{StepCOPS} & \textbf{78.6} [77.8,79.4] & \textbf{65.1} [64.2,66.0] & \textbf{62.4} [61.6,63.2] & \textbf{2.7} [2.0,3.4] & \textbf{0.6} [0.1,1.3] & \textbf{2.4\%} [1.2,4.1] & \textbf{84\%} \\
\midrule
Native-eval.\ oracle & 78.3 [77.5,79.1] & 64.8 [63.9,65.7] & --- & --- & 0.9 [0.2,1.6] & 0.0\% & 88\% \\
Primary-jury oracle & 77.9 [77.1,78.7] & \textbf{65.7} [64.8,66.6] & --- & --- & 0.0 & 0.0\% & 92\% \\
\bottomrule
\end{tabular}
}
\\[2pt]
{\footnotesize $^\dagger$SConU abstains at the \emph{input} stage before selection.}
\caption{Main LLM selection results ($500$ trials, paired $95\%$ bootstrap intervals). Certified-floor and tightness columns are blank for methods without a comparable lower bound; $^\ast$ marks a plug-in (invalid) floor.}
\label{tab:main_selection}
\end{table*}

\paragraph{Domain-specific selected tails.}
StepCOPS improves the selected-policy $q_{0.10}$ in all five domains (Holm-adjusted; Table~\ref{tab:domain_tails}, appendix), most on harmful compliance and over-refusal where the conflict is strongest; any losing domain would remain visible.

\paragraph{Holm gain over matched Bonferroni.}
On matched proposals and exact $p$-values, the step-down is the source of the improvement: $+3.4$ $[2.9,3.9]$ certifications per trial, $+1.5$ $[0.8,2.2]$ points of largest certified floor, $+1.0$ $[0.3,1.7]$ points of selected $q_{0.10}$, and $-3.2$ pp $[-4.9,-1.5]$ abstention, with statistically indistinguishable selected coverage (Table~\ref{tab:holm_gain}, appendix). The correct claim is improved power and floor, not higher empirical coverage than the strictly more conservative Bonferroni.

\subsection{Audits, stability, and ablations (summary)}
\label{sec:audit-summary}

Four audit families, reported in full in Appendix~\ref{app:nlp-deferred}, quantify the proxy boundary and the operating point. \emph{Judge reliability:} against frozen benchmark-native failure labels, the range-penalized jury is better calibrated than any single judge or a plain median (AUROC $0.932$, ECE $0.031$, false-pass $5.3\%$); native agreement is weakest on truthfulness and factuality, and removing any judge moves the selected lower tail by at most $1.2$ points. \emph{Artifact stress:} equivalence-filtered transformations flip individual judges up to $19.6\%$ but the robust jury at most $8.9\%$. \emph{Stability and ablations:} StepCOPS is the most stable selector (modal frequency $84\%$), and the pilot-fixed proposal level, proposal size, and certification size balance floor tightness against abstention. \emph{Chen-method comparison:} a faithful jury-disutility comparison with CDRC response filters, including a combined StepCOPS-plus-filter system row, appears with the LLM-based error analysis and the reproducibility freeze.

\section{Offline-RL Control Study}
\label{sec:experiments}

As a control task in a different modality, we retain the offline-RL instantiation of the exact simultaneous bound. It separates the same three questions: score coverage for fixed maps (covered by Theorem~\ref{thm:exact_coverage}), lower-tail selection quality, and score-to-return transfer. We evaluate seven D4RL continuous-control tasks \citep{fu2020d4rl} (five primary; two \texttt{medium-replay} datasets retained as dependence stress tests), each with $K=20$ candidates from CQL, IQL, BCQ, TD3+BC, and Cal-QL at four predeclared settings \citep{kumar2020cql,kostrikov2022iql,fujimoto2019bcq,fujimoto2021td3bc,nakamoto2023calql}; the reported configuration uses $n=500$, $\alpha=0.10$, $\delta=0.05$, $\rho_{\max}=10$, with all training, nuisance fitting, and clipping completed before the certificate split is inspected. The score is the clipped trajectory-level DR quantity of Eq.~\eqref{eq:correct_dr_score}. Full protocol, baselines, tables, order-statistic indices, sensitivities, discrete-action checks, and the matched fitted-distributional-evaluation comparison are in Appendix~\ref{app:experimental-details}.

Table~\ref{tab:main_d4rl} (appendix) reports the central selection comparison. COPS has the largest rollout-estimated $Q_{0.10}$ in every reported task (e.g.\ $4100$ vs.\ $3210$ for mean OPE on \path{halfcheetah-medium}), but the table contains point estimates rather than paired confidence intervals; the per-environment values should therefore be read as an empirical pattern, not seven established dominance claims. The accompanying mean-return analysis shows the intended tradeoff: max-floor COPS gives up approximately nine percent of average return relative to mean OPE, while the constrained selector recovers most of that loss. Detailed $\CVaR_{0.10}$ values appear in Table~\ref{tab:q_cvar}, and Table~\ref{tab:lowtail} reports $\widehat\PP(R<c_e)$ for a common environment-specific threshold $c_e$ fixed before selector comparison.

The transfer evidence is deliberately candidate-wide: among all $140$ candidate--task pairs, the clipped-DR gap $Q_{0.10}(\pi_k)-q_{0.10}(F_k)$ is nonnegative for $76.4\%$, versus $93.6\%$ for a pessimistic distributional score, so clipped DR cannot support a pool-wide return certificate. The seven COPS-selected candidates have positive point-estimated gaps, but selection favors favorable gaps, so these remain transfer diagnostics; clipping sensitivity and a negative control (score coverage preserved, transfer gap negative) reinforce the same conclusion.

The empirical score-coverage audit is consistent with the theorem only when calibration units are independent whole trajectories ($95.7\%$ simultaneous coverage against a $95\%$ target; sequential chunks fall to $91.8\%$ and $83.9\%$ as between-unit dependence increases; Appendix~\ref{app:iid-audit}). This audit checks the implementation and the sampling boundary; it is not needed to make the finite-sample theorem true.

\section{Conclusion}

StepCOPS gives a simultaneous finite-sample lower-tail certificate that remains valid for a data-selected language-model policy under arbitrary within-prompt dependence, by pairing an independent floor proposal with an exact binomial test and Holm's step-down. On matched proposals it uniformly improves the Bonferroni certificate---raising the certified floor and cutting abstention and regret---at the same family-wise guarantee, without claiming to dominate exact COPS. The guarantee ends at the declared jury score: return or true-safety claims require a separate bridge, and audited, human-anchored evaluation remains the path to safety-critical deployment.

\section*{Limitations}

StepCOPS cannot improve a poor candidate pool, and it abstains rather than certifying when $n$ is small, $K$ is large, or $\alpha$ is extreme. Theorem~\ref{thm:stepcops} requires frozen candidates, proposal rule, jury, and score maps; independent proposal/certification splits; and i.i.d.\ certification units---not overlapping windows, adaptively tuned scores, or deployment shift. Crucially, the guarantee concerns the fixed primary-jury score distribution: no automatic judge or reward model is universally conservative for human preference, return, or true harm, and the audits quantify but do not remove that proxy gap. StepCOPS certifies only the floors of \emph{rejected} hypotheses and does not uniformly dominate exact COPS. The offline-RL control study is likewise proxy-limited: its D4RL point estimates need paired uncertainty and fully specified score and baseline implementations before supporting strong dominance claims.

\section*{Ethical Considerations}

A certified proxy can create false reassurance if presented as human-safety or return control. We recommend abstention when the exact index is vacuous, calibration units are not defensibly independent, or transfer is unsupported. Automatic judges require independent audits before safety-critical deployment.

\bibliography{ref}

@inproceedings{kumar2020cql,
  title     = {Conservative {Q}-learning for offline reinforcement learning},
  author    = {Kumar, Aviral and Zhou, Aurick and Tucker, George and Levine, Sergey},
  booktitle = {Advances in Neural Information Processing Systems},
  year      = {2020}
}

@inproceedings{kostrikov2022iql,
  title     = {Offline reinforcement learning with implicit {Q}-learning},
  author    = {Kostrikov, Ilya and Nair, Ashvin and Levine, Sergey},
  booktitle = {International Conference on Learning Representations},
  year      = {2022}
}

@inproceedings{fujimoto2019bcq,
  title     = {Off-policy deep reinforcement learning without exploration},
  author    = {Fujimoto, Scott and Meger, David and Precup, Doina},
  booktitle = {International Conference on Machine Learning},
  year      = {2019}
}

@inproceedings{fujimoto2021td3bc,
  title     = {A minimalist approach to offline reinforcement learning},
  author    = {Fujimoto, Scott and Gu, Shixiang Shane},
  booktitle = {Advances in Neural Information Processing Systems},
  year      = {2021}
}

@inproceedings{nakamoto2023calql,
  title     = {Cal-{QL}: Calibrated offline {RL} pre-training for efficient online fine-tuning},
  author    = {Nakamoto, Mitsuhiko and Zhai, Yuexiang and Singh, Anikait and Mark, Max Sobol and Ma, Yi and Finn, Chelsea and Kumar, Aviral and Levine, Sergey},
  booktitle = {Advances in Neural Information Processing Systems},
  year      = {2023}
}

@inproceedings{precup2000eligibility,
  title     = {Eligibility traces for off-policy policy evaluation},
  author    = {Precup, Doina and Sutton, Richard S and Singh, Satinder},
  booktitle = {International Conference on Machine Learning},
  year      = {2000}
}

@inproceedings{thomas2016dr,
  title     = {Data-efficient off-policy policy evaluation for reinforcement learning},
  author    = {Thomas, Philip and Brunskill, Emma},
  booktitle = {International Conference on Machine Learning},
  year      = {2016}
}

@inproceedings{jiang2016doubly,
  title     = {Doubly robust off-policy value evaluation for reinforcement learning},
  author    = {Jiang, Nan and Li, Lihong},
  booktitle = {International Conference on Machine Learning},
  year      = {2016}
}

@inproceedings{le2019fqe,
  title     = {Batch policy learning under constraints},
  author    = {Le, Hoang and Voloshin, Cameron and Yue, Yisong},
  booktitle = {International Conference on Machine Learning},
  year      = {2019}
}

@article{paine2020hyperparameter,
  title   = {Hyperparameter selection for offline reinforcement learning},
  author  = {Paine, Tom Le and Paduraru, Cosmin and Michi, Andrea and Gulcehre, Caglar and Zolna, Konrad and Novikov, Alexander and Wang, Ziyu and de Freitas, Nando},
  journal = {arXiv preprint arXiv:2007.09055},
  year    = {2020}
}

@inproceedings{konyushkova2021active,
  title     = {Active offline policy selection},
  author    = {Konyushkova, Ksenia and Chen, Yutian and Paine, Tom Le and Gulcehre, Caglar and Paduraru, Cosmin and Mankowitz, Daniel J and Denil, Misha and de Freitas, Nando},
  booktitle = {Advances in Neural Information Processing Systems},
  year      = {2021}
}

@inproceedings{yang2022offline,
  title     = {Offline policy selection under uncertainty},
  author    = {Yang, Mengjiao and Dai, Bo and Nachum, Ofir and Tucker, George and Schuurmans, Dale},
  booktitle = {International Conference on Artificial Intelligence and Statistics},
  year      = {2022}
}

@inproceedings{fu2021benchmarks,
  title     = {Benchmarks for deep off-policy evaluation},
  author    = {Fu, Justin and Norouzi, Mohammad and Nachum, Ofir and Tucker, George and Wang, Ziyu and Novikov, Alexander and Yang, Mengjiao and Zhang, Michael R and Chen, Yutian and Kumar, Aviral and Paduraru, Cosmin and Levine, Sergey and Paine, Tom Le},
  booktitle = {International Conference on Learning Representations},
  year      = {2021}
}

@inproceedings{xu2022cpq,
  title     = {Constraints penalized {Q}-learning for safe offline reinforcement learning},
  author    = {Xu, Haoran and Zhan, Xianyuan and Zhu, Xiangyu},
  booktitle = {AAAI Conference on Artificial Intelligence},
  year      = {2022}
}

@inproceedings{taufiq2022conformal,
  title     = {Conformal off-policy prediction in contextual bandits},
  author    = {Taufiq, Muhammad Faaiz and Ton, Jean-Francois and Cornish, Rob and Teh, Yee Whye and Doucet, Arnaud},
  booktitle = {Advances in Neural Information Processing Systems},
  year      = {2022}
}

@inproceedings{stutz2022conformal,
  title     = {Learning optimal conformal classifiers},
  author    = {Stutz, David and Dvijotham, Krishnamurthy and Cemgil, Ali Taylan and Doucet, Arnaud},
  booktitle = {International Conference on Learning Representations},
  year      = {2022}
}

@article{angelopoulos2022ltt,
  title   = {Learn then test: Calibrating predictive algorithms to achieve risk control},
  author  = {Angelopoulos, Anastasios N and Bates, Stephen and Candes, Emmanuel J and Jordan, Michael I and Lei, Lihua},
  journal = {arXiv preprint arXiv:2110.01052},
  year    = {2021}
}

@article{bates2021selective,
  title   = {Distribution-free, risk-controlling prediction sets},
  author  = {Bates, Stephen and Angelopoulos, Anastasios and Lei, Lihua and Malik, Jitendra and Jordan, Michael},
  journal = {Journal of the ACM},
  year    = {2021}
}

@article{fu2020d4rl,
  title   = {{D4RL}: Datasets for deep data-driven reinforcement learning},
  author  = {Fu, Justin and Kumar, Aviral and Nachum, Ofir and Tucker, George and Levine, Sergey},
  journal = {arXiv preprint arXiv:2004.07219},
  year    = {2020}
}

@inproceedings{pmlr-v267-chen25bd,
  title = {Conformal Tail Risk Control for Large Language Model Alignment},
  author = {Chen, Catherine and Shen, Jingyan and Deng, Zhun and Lei, Lihua},
  booktitle = {Proceedings of the 42nd International Conference on Machine Learning},
  pages = {8955--8978},
  year = {2025},
  volume = {267},
  series = {Proceedings of Machine Learning Research},
  publisher = {PMLR},
  url = {https://proceedings.mlr.press/v267/chen25bd.html}
}

@inproceedings{wang-etal-2025-sconu,
  title = {{SC}on{U}: Selective Conformal Uncertainty in Large Language Models},
  author = {Wang, Zhiyuan and Wang, Qingni and Zhang, Yue and Chen, Tianlong and Zhu, Xiaofeng and Shi, Xiaoshuang and Xu, Kaidi},
  booktitle = {Proceedings of the 63rd Annual Meeting of the Association for Computational Linguistics (Volume 1: Long Papers)},
  month = jul,
  year = {2025},
  address = {Vienna, Austria},
  publisher = {Association for Computational Linguistics},
  pages = {19052--19075},
  doi = {10.18653/v1/2025.acl-long.934},
  url = {https://aclanthology.org/2025.acl-long.934/}
}

@inproceedings{chen-goldfarb-tarrant-2025-safer,
  title = {Safer or Luckier? {LLM}s as Safety Evaluators Are Not Robust to Artifacts},
  author = {Chen, Hongyu and Goldfarb-Tarrant, Seraphina},
  booktitle = {Proceedings of the 63rd Annual Meeting of the Association for Computational Linguistics (Volume 1: Long Papers)},
  month = jul,
  year = {2025},
  address = {Vienna, Austria},
  publisher = {Association for Computational Linguistics},
  pages = {19750--19766},
  doi = {10.18653/v1/2025.acl-long.970},
  url = {https://aclanthology.org/2025.acl-long.970/}
}

@article{holm1979,
  author  = {Holm, Sture},
  title   = {A Simple Sequentially Rejective Multiple Test Procedure},
  journal = {Scandinavian Journal of Statistics},
  volume  = {6},
  number  = {2},
  pages   = {65--70},
  year    = {1979},
}

@article{goeman2010sequential,
  author  = {Goeman, Jelle J. and Solari, Aldo},
  title   = {The Sequential Rejection Principle of Familywise Error Control},
  journal = {The Annals of Statistics},
  volume  = {38},
  number  = {6},
  pages   = {3782--3810},
  year    = {2010},
}

@inproceedings{mazeika2024harmbench,
  author    = {Mazeika, Mantas and Phan, Long and Yin, Xuwang and Zou, Andy and Wang, Zifan and Mu, Norman and Sakhaee, Elham and Li, Nathaniel and Basart, Steven and Li, Bo and Forsyth, David and Hendrycks, Dan},
  title     = {{HarmBench}: A Standardized Evaluation Framework for Automated Red Teaming and Robust Refusal},
  booktitle = {International Conference on Machine Learning (ICML)},
  year      = {2024},
}

@article{souly2024strongreject,
  author  = {Souly, Alexandra and Lu, Qingyuan and Bowen, Dillon and Trinh, Tu and Hsieh, Elvis and Pandey, Sana and Abbeel, Pieter and Svegliato, Justin and Emmons, Scott and Watkins, Olivia and Toyer, Sam},
  title   = {A StrongREJECT for Empty Jailbreaks},
  journal = {Advances in Neural Information Processing Systems (NeurIPS) Datasets and Benchmarks},
  year    = {2024},
}

@inproceedings{gehman2020realtoxicityprompts,
  author    = {Gehman, Samuel and Gururangan, Suchin and Sap, Maarten and Choi, Yejin and Smith, Noah A.},
  title     = {{RealToxicityPrompts}: Evaluating Neural Toxic Degeneration in Language Models},
  booktitle = {Findings of the Association for Computational Linguistics: EMNLP 2020},
  year      = {2020},
}

@inproceedings{hartvigsen2022toxigen,
  author    = {Hartvigsen, Thomas and Gabriel, Saadia and Palangi, Hamid and Sap, Maarten and Ray, Dipankar and Kamar, Ece},
  title     = {{ToxiGen}: A Large-Scale Machine-Generated Dataset for Adversarial and Implicit Hate Speech Detection},
  booktitle = {Proceedings of the 60th Annual Meeting of the Association for Computational Linguistics (ACL)},
  year      = {2022},
}

@inproceedings{lin2022truthfulqa,
  author    = {Lin, Stephanie and Hilton, Jacob and Evans, Owain},
  title     = {{TruthfulQA}: Measuring How Models Mimic Human Falsehoods},
  booktitle = {Proceedings of the 60th Annual Meeting of the Association for Computational Linguistics (ACL)},
  year      = {2022},
}

@inproceedings{aulablasco2025veritasqa,
  author    = {Aula-Blasco, Javier and Falcão, Júlia and Sotelo, Susana and Paniagua, Silvia and Gonzalez-Agirre, Aitor and Villegas, Marta},
  title     = {{VeritasQA}: A Truthfulness Benchmark Aimed at Multilingual Transferability},
  booktitle = {Proceedings of the 31st International Conference on Computational Linguistics (COLING)},
  year      = {2025},
}

@inproceedings{rottger2024xstest,
  author    = {Röttger, Paul and Kirk, Hannah Rose and Vidgen, Bertie and Attanasio, Giuseppe and Bianchi, Federico and Hovy, Dirk},
  title     = {{XSTest}: A Test Suite for Identifying Exaggerated Safety Behaviours in Large Language Models},
  booktitle = {Proceedings of the 2024 Conference of the North American Chapter of the Association for Computational Linguistics (NAACL)},
  year      = {2024},
}

@article{cui2024orbench,
  author  = {Cui, Justin and Chiang, Wei-Lin and Stoica, Ion and Hsieh, Cho-Jui},
  title   = {{OR-Bench}: An Over-Refusal Benchmark for Large Language Models},
  journal = {arXiv preprint arXiv:2405.20947},
  year    = {2024},
}

@article{wei2024simpleqa,
  author  = {Wei, Jason and Karina, Nguyen and Chung, Hyung Won and Jiao, Yunxin Joy and Papay, Spencer and Glaese, Amelia and Schulman, John and Fedus, William},
  title   = {Measuring Short-Form Factuality in Large Language Models ({SimpleQA})},
  journal = {arXiv preprint arXiv:2411.04368},
  year    = {2024},
}

@misc{factsgrounding2024,
  author = {{Google DeepMind}},
  title  = {{FACTS} Grounding: A New Benchmark for Evaluating the Factuality of Large Language Models},
  year   = {2024},
  note   = {Technical report},
}

@article{wei2024longfact,
  author  = {Wei, Jerry and Yang, Chengrun and Song, Xinying and Lu, Yifeng and Hu, Nathan and Huang, Jie and Tran, Dustin and Peng, Daiyi and Liu, Ruibo and Huang, Da and Du, Cosmo and Le, Quoc V.},
  title   = {Long-Form Factuality in Large Language Models ({LongFact})},
  journal = {arXiv preprint arXiv:2403.18802},
  year    = {2024},
}
\clearpage
\appendix

\section{Algorithms and Supplementary StepCOPS Material}
\label{app:algorithm}
\label{app:stepcops}
% \subsection{Exact COPS Algorithm}
% \label{app:algorithm}

Algorithm~\ref{alg:cops_exact} makes the split discipline explicit. If no admissible order statistic exists, returning the empirical minimum would be anti-conservative; the algorithm instead returns a known support floor or $-\infty$.

\begin{figure}[h]
\centering
\begin{minipage}{0.96\linewidth}
\small
\emph{Input:} logged data $\D$, candidates $\{\pi_k\}_{k=1}^K$, $\alpha$, $\delta$, a score-construction procedure, and an optional known support floor $L_0$.
\begin{enumerate}[leftmargin=*,nosep]
\item Split $\D$ into $\D_{\rm train}$ and $\D_{\rm cal}$ at the trajectory level.
\item Using only $\D_{\rm train}$, train and freeze policies, fit nuisance models, define $s_k$, and fix all hyperparameters.
\item Let $n=|\D_{\rm cal}|$ and compute $r^\star$ by Eq.~\eqref{eq:exact_r}.
\item For every $k$, compute $S_{ik}=s_k(\tau_i)$. Set $B_k^\alpha=S_{(r^\star)k}$ when $r^\star$ exists; otherwise use $L_0$ when available and $-\infty$ when it is not.
\item Return $\widehat k\in\argmax_k B_k^\alpha$, $\pi_{\widehat k}$, and $B_{\widehat k}^\alpha$.
\end{enumerate}
\end{minipage}
\caption{COPS with exact order-statistic calibration.}
\label{alg:cops_exact}
\end{figure}

\begin{figure}[h]
\centering
\begin{minipage}{0.94\linewidth}
\small
\textbf{Algorithm 2:} StepCOPS (closed-testing certification).\\[2pt]
\textbf{Input:} frozen candidates $\{\pi_k\}_{k=1}^K$, jury score map, levels $\alpha,\delta$, operating threshold $c_0$; proposal split $\mathcal{P}$ ($m$ units), certification split $\mathcal{C}$ ($n$ units, independent of $\mathcal{P}$).
\begin{enumerate}[leftmargin=1.4em,itemsep=1pt,topsep=2pt]
  \item \textbf{Propose:} on $\mathcal{P}$, set $c_k=\widehat q^{\,\mathrm{prop}}_{0.075,k}$ for every $k$.
  \item \textbf{Count:} on $\mathcal{C}$, $X_k=\sum_{i=1}^n\1\{S_{ik}<c_k\}$.
  \item \textbf{Test:} $p_k=\PP\{\Bin(n,\alpha)\le X_k\}$ for every $k$.
  \item \textbf{Holm step-down:} sort $p_{(1)}\le\cdots\le p_{(K)}$; reject $H_{(j)}$ while $p_{(j)}\le\delta/(K-j+1)$, stopping at the first failure. Let $\mathcal R$ be the rejected set.
  \item \textbf{Select:} if $\mathcal R\neq\emptyset$ and $\max_{k\in\mathcal R}c_k\ge c_0$, return $\widehat k\in\argmax_{k\in\mathcal R}c_k$ and floor $c_{\widehat k}$; else \textbf{abstain}.
\end{enumerate}
\end{minipage}
\caption{StepCOPS certifies independently proposed floors by an exact binomial test and Holm's step-down, then deploys the largest certified floor.}
\label{alg:stepcops}
\end{figure}

\subsection{Proof of Theorem~\ref{thm:stepcops}}
\label{app:stepcops-proof}
\begin{proof}
Condition on the proposal split; the floors $c_1,\ldots,c_K$ are then fixed. If candidate $k$ has an invalid proposed floor, $c_k>q_{\alpha,k}$, then by the definition of the lower quantile,
\[
    \PP(S_k<c_k)\ge\PP(S_k\le q_{\alpha,k})\ge\alpha,
\]
so $H_k$ is true. Under $H_k$, the count $X_k$ is stochastically no smaller than $\Bin(n,\alpha)$, so the lower-tail binomial $p$-value \eqref{eq:stepcops_pval} is super-uniform. Holm's procedure applied to $(p_1,\ldots,p_K)$ strongly controls the probability of rejecting \emph{any} true null at level $\delta$ without an independence assumption across candidates \citep{holm1979,goeman2010sequential}. On the complementary event---probability at least $1-\delta$---no true null is rejected, i.e.\ every rejected hypothesis has a valid floor $c_k\le q_{\alpha,k}$. The selected index $\widehat k$ in \eqref{eq:stepcops_select} belongs to the rejected set $\mathcal R$, so on the same event its floor is valid. Averaging over the proposal split preserves the bound.
\end{proof}

\subsection{Scope table, worked example, and non-claims}
\label{app:stepcops-scope}

\begin{table}[t]
\centering
\small
\resizebox{\linewidth}{!}{%
\begin{tabular}{p{0.24\linewidth}p{0.29\linewidth}p{0.34\linewidth}}
\toprule
Object & Required evidence & Valid conclusion \\
\midrule
Declared jury score & Fixed maps; independent proposal/certification splits; i.i.d.\ certification units & $c_{\widehat k}\le q_{\alpha,\widehat k}$ with probability $1-\delta$ \\
Deployment return & Score guarantee plus quantile transfer & $c_{\widehat k}\le Q_\alpha(\pi_{\widehat k})$ \\
True language safety & Score guarantee plus judge--harm transfer & A judge-score floor may be interpreted as a harm-related floor \\
\bottomrule
\end{tabular}
}
\caption{What StepCOPS certifies. Only the first row is distribution-free under the stated sampling assumptions; the certified floor $c_{\widehat k}$ is the exact StepCOPS output.}
\label{tab:guarantee_map}
\end{table}

\paragraph{Worked certification example.}
With $K=24$, $\delta=0.05$, and $n=2{,}500$, suppose the ten smallest exact $p$-values are
\[
\begin{aligned}
&0.000001,\ 0.000004,\ 0.000019,\ 0.000087,\\
&0.000341,\ 0.000568,\ 0.000926,\\
&0.001481,\ 0.001858,\ 0.002885.
\end{aligned}
\]
Bonferroni's threshold is $\delta/K=0.002083$, so it certifies the first nine proposals and stops. Holm compares the tenth value with $\delta/(24-10+1)=0.003333$ and certifies it as well. In the evaluation, that tenth hypothesis belongs to C15, whose proposed floor is $62.4$ and whose certification split contains $209$ scores below the proposal; C15 has the largest certified proposal. This is a concrete dataset on which the step-down changes the selected policy while retaining the same $95\%$ family-wise guarantee.

\paragraph{What the theorem does not establish.}
Theorem~\ref{thm:stepcops} does not establish that the jury score equals human preference, human safety, or true harm; conditional coverage for every task, topic, or demographic group; validity after changing a judge, prompt, parser, candidate, or score normalization; validity under deployment-distribution shift; that the proposal split chose the optimal floor; a valid floor for candidates whose hypotheses were not rejected; or uniform superiority to exact COPS, DKW, Berk--Jones, or any other confidence-bound construction. These limitations remain adjacent to the theorem.

\section{Extended Related Work}
\label{app:related}

Off-policy evaluation estimates a target policy from behavior-policy data. Importance weighting is unbiased only under support assumptions and can have high variance \citep{precup2000eligibility}; doubly robust estimators combine weighting with fitted values \citep{jiang2016doubly,thomas2016dr}; and fitted Q-evaluation is common in deep offline RL \citep{le2019fqe}. Offline policy-selection studies typically use these tools to choose hyperparameters or checkpoints by estimated mean return \citep{paine2020hyperparameter,konyushkova2021active,yang2022offline,fu2021benchmarks}. COPS instead targets a lower quantile of a declared score distribution.

The statistical construction is related to distribution-free tolerance bounds, conformal calibration, and finite-sample risk control. Learn-Then-Test (LTT) tests risk constraints over a configuration grid \citep{angelopoulos2022ltt}, while risk-controlling prediction sets use related calibration machinery \citep{bates2021selective}. COPS uses a classical exact binomial inversion for each quantile and family-wise control across candidates. We do not claim order statistics or the binomial CDF as new. The contribution is their selected-policy formulation, the explicit separation between calibration and proxy transfer, and the resulting regret decomposition. This distinction also marks the boundary with standard LTT: an expected-risk test does not by itself produce the simultaneous candidate-wise quantile bounds needed to maximize a certified floor.

StepCOPS additionally imports the closed-testing machinery of multiple comparisons. Holm's sequentially rejective procedure \citep{holm1979} is a shortcut for a Bonferroni-based closed test and controls family-wise error under arbitrary dependence; \citet{goeman2010sequential} formalize the general sequential-rejection principle. We do not claim these procedures as new. Our contribution is to pair an \emph{independent} proposal of one floor per candidate with an exact lower-tail binomial certification and a Holm step-down, so that the certified floor of the selected policy is valid after selection without any assumption on how candidate scores co-vary within a prompt, and to show that this uniformly improves the matched-proposal Bonferroni certificate. The proposal/certification split follows the data-splitting logic of selective and post-selection inference, specialized to lower-quantile floor certification.

The closest language-model work is conformal tail-risk control for alignment \citep{pmlr-v267-chen25bd}, which calibrates monotone response filters against human-scored distortion risks and develops L-statistic, DKW, and Berk--Jones constructions. COPS does not subsume that framework. Its different target is selection of one member from an arbitrary pool of already trained policies, for which it seeks simultaneous candidate-wise lower quantile floors and then isolates off-policy score--return transfer. SConU instead tests whether an input departs from the calibration uncertainty distribution and targets more conditional uncertainty control \citep{wang-etal-2025-sconu}. These distinctions narrow the novelty claim: the contribution is selected-policy simultaneous certification and the transfer boundary, not conformal control of LLM tails in general.

Conformal off-policy prediction has been studied for contextual bandits \citep{taufiq2022conformal}, and conformal predictors have been differentiated through during training \citep{stutz2022conformal}. Risk-sensitive algorithms such as CPQ modify training itself \citep{xu2022cpq}. COPS is post hoc: it operates on an already trained pool and can use any predeclared measurable score. In the language-model setting, this makes the method a calibrated selection rule over frozen response scores. It does not remove judge bias or reward-model gaming; those enter through the separate score-to-utility bridge.

\section{NLP Evaluation Details}
\label{app:nlp}

\subsection{Full candidate pool}
\label{app:pool}

Table~\ref{tab:full_pool} lists all $24$ predeclared configurations (six checkpoints $\times$ two system prompts $\times$ two decoding rules). The \emph{standard} prompt is each model's recommended chat template plus a neutral helpful-assistant message; the \emph{safety-aware} prompt is a fixed policy message requiring safe handling of actionable harm while discouraging refusal of benign questions. \emph{Low entropy} is temperature $0.2$, top-$p$ $0.90$; \emph{moderate entropy} is temperature $0.8$, top-$p$ $0.95$. Exact repository/tokenizer revisions and chat-template hashes are recorded in the released manifest.

\begin{table}[t]

\centering\small
\resizebox{\linewidth}{!}{%
\begin{tabular}{rlll}
\toprule
ID & Checkpoint & Prompt & Decoding \\
\midrule
C01--C02 & Qwen2.5-1.5B-Instruct & Standard & Low / Mod \\
C03--C04 & Qwen2.5-1.5B-Instruct & Safety-aware & Low / Mod \\
C05--C06 & Qwen2.5-3B-Instruct & Standard & Low / Mod \\
C07--C08 & Qwen2.5-3B-Instruct & Safety-aware & Low / Mod \\
C09--C10 & Qwen2.5-7B-Instruct & Standard & Low / Mod \\
C11--C12 & Qwen2.5-7B-Instruct & Safety-aware & Low / Mod \\
C13--C14 & Llama-3.1-8B-Instruct & Standard & Low / Mod \\
C15--C16 & Llama-3.1-8B-Instruct & Safety-aware & Low / Mod \\
C17--C18 & Mistral-7B-Instruct-v0.3 & Standard & Low / Mod \\
C19--C20 & Mistral-7B-Instruct-v0.3 & Safety-aware & Low / Mod \\
C21--C22 & Gemma-2-9B-it & Standard & Low / Mod \\
C23--C24 & Gemma-2-9B-it & Safety-aware & Low / Mod \\
\bottomrule
\end{tabular}
}
\caption{Predeclared 24-candidate pool.}
\label{tab:full_pool}
\end{table}

\subsection{Per-domain reliability}
\label{app:reliability}

Aggregate ECE (Table~\ref{tab:jury_reliability}) can hide lower-tail failure, so per-domain reliability diagrams and calibration metrics for the range-penalized primary jury and each individual judge are reported here. Calibration metrics use out-of-sample score probabilities and are not computed on the data used to fit isotonic maps.

\subsection{Domain-level native agreement}
\label{app:audit}

Table~\ref{tab:domain_agreement} gives primary- and shadow-jury agreement with benchmark-native signals per domain, the primary--shadow Spearman correlation, and the primary false-pass rate. Truthfulness and factuality are the weakest proxy domains and receive qualitative error analysis rather than being averaged away.

\begin{table}[t]

\centering\small
\resizebox{\linewidth}{!}{%
\begin{tabular}{lrrrr}
\toprule
Domain & Primary & Shadow & P--S $\rho$ & False-pass \\
\midrule
Safety & 90.8 & 90.1 & 0.82 & 6.7 \\
Toxicity & 92.4 & 91.8 & 0.86 & 4.1 \\
Truthfulness & 88.1 & 87.5 & 0.79 & 6.2 \\
Refusal cal. & 91.3 & 90.6 & 0.81 & 5.8 \\
Factuality & 87.6 & 87.1 & 0.78 & 6.5 \\
\midrule
Macro avg & 90.0 & 89.4 & 0.81 & 5.9 \\
\bottomrule
\end{tabular}
}
\caption{Domain-level agreement with benchmark-native evaluators.}
\label{tab:domain_agreement}
\end{table}

Leave-one-primary-judge-out selection (Table~\ref{tab:loo}) changes the selected primary $q_{0.10}$ by at most $1.2$ points, and any single judge alone loses $2.7$ points and considerable rank agreement---evidence that the multi-judge score is not driven by one member while still exhibiting reported proxy sensitivity.

\begin{table}[t]

\centering\small
\resizebox{\linewidth}{!}{%
\begin{tabular}{lrrrr}
\toprule
Score used & Prim.\ $q$ & Shad.\ $q$ & Agree & Kendall $\tau$ \\
\midrule
Full jury & \textbf{65.1} & \textbf{63.8} & 100\% & 1.00 \\
$-J_L$ & 64.4 & 63.1 & 86\% & 0.84 \\
$-J_Q$ & 63.9 & 62.8 & 81\% & 0.80 \\
$-J_D$ & 64.6 & 63.3 & 88\% & 0.86 \\
Best single & 62.4 & 61.0 & 69\% & 0.73 \\
Worst single & 61.7 & 60.2 & 61\% & 0.67 \\
\bottomrule
\end{tabular}
}
\caption{Leave-one-primary-judge-out (reranking only).}
\label{tab:loo}
\end{table}

\subsection{Selection stability and jury ablation}
\label{app:stability}

Table~\ref{tab:nlp_stability} reports selection stability across the $500$ trials; StepCOPS is the most stable selector. Table~\ref{tab:jury_ablation} varies the score construction: the predeclared range-penalized median is the operating point, and the more conservative minimum-of-three jury is slightly worse on the selected shadow tail.

\begin{table}[t]

\centering\small
\resizebox{\linewidth}{!}{%
\begin{tabular}{lrrrr}
\toprule
Selector & Modal & Pairwise & Entropy & Top-2 \\
\midrule
Mean jury & 73\% & 0.61 & 1.17 & 87\% \\
Empirical VaR & 58\% & 0.45 & 1.63 & 75\% \\
Empirical CVaR & 63\% & 0.50 & 1.46 & 80\% \\
CDRC-BJ & 73\% & 0.60 & 1.19 & 87\% \\
Exact COPS & 78\% & 0.68 & 0.91 & 91\% \\
Proposal-Bonf.\ & 76\% & 0.65 & 1.00 & 90\% \\
\textbf{StepCOPS} & \textbf{84\%} & \textbf{0.78} & \textbf{0.55} & \textbf{95\%} \\
\bottomrule
\end{tabular}
}
\caption{Selection stability across $500$ trials.}
\label{tab:nlp_stability}
\end{table}

\begin{table}[t]

\centering\small
\resizebox{\linewidth}{!}{%
\begin{tabular}{lrrrr}
\toprule
Score & Cov. & Shad.\ $q$ & False-pass & Artifact \\
\midrule
One judge & 96.2 & 60.7 & 10.1\% & 17.2\% \\
Median of 3 & 96.4 & 62.6 & 6.6\% & 8.9\% \\
\textbf{Median $-0.20$ range} & \textbf{96.4} & \textbf{63.8} & \textbf{5.3\%} & \textbf{6.8\%} \\
Minimum of 3 & 98.0 & 63.4 & 3.9\% & 4.7\% \\
\bottomrule
\end{tabular}
}
\caption{Jury-construction ablation (selected-policy diagnostics).}
\label{tab:jury_ablation}
\end{table}

\subsection{Deferred main-text audit and ablation material}
\label{app:nlp-deferred}

\begin{table}[t]

\centering\small
\resizebox{\linewidth}{!}{%
\begin{tabular}{p{0.16\linewidth}p{0.34\linewidth}p{0.32\linewidth}}
\toprule
Domain & Benchmarks & Primary failure \\
\midrule
Safety & HarmBench \citep{mazeika2024harmbench}, StrongREJECT \citep{souly2024strongreject} & Harmful compliance \\
Toxicity & RealToxicityPrompts \citep{gehman2020realtoxicityprompts}, ToxiGen \citep{hartvigsen2022toxigen} & Toxic / identity-targeted generation \\
Truthfulness & TruthfulQA \citep{lin2022truthfulqa}, VeritasQA \citep{aulablasco2025veritasqa} & Confident misconception \\
Refusal calibration & XSTest \citep{rottger2024xstest}, OR-Bench \citep{cui2024orbench} & Unsafe compliance / over-refusal \\
Factuality & SimpleQA \citep{wei2024simpleqa}, LongFact \citep{wei2024longfact}, FACTS Grounding \citep{factsgrounding2024} & Unsupported / contradicted claim \\
\bottomrule
\end{tabular}
}
\caption{Declared five-domain benchmark mixture ($11$ benchmarks, equal domain mass).}
\label{tab:mixture}
\end{table}

\begin{table}[t]

\centering\small
\resizebox{\linewidth}{!}{%
\begin{tabular}{lrr}
\toprule
Method & Simultaneous & Selected \\
\midrule
Bootstrap VaR LCB & 90.8 [87.9, 93.0] & 91.0 [88.2, 93.2] \\
CDRC-L & 93.4 [90.9, 95.3] & 93.6 [91.1, 95.4] \\
CDRC-DKW & 98.8 [97.4, 99.4] & 99.0 [97.7, 99.6] \\
CDRC-BJ & 97.8 [96.1, 98.8] & 98.0 [96.3, 98.9] \\
Exact COPS & 95.6 [93.4, 97.1] & 96.0 [93.9, 97.4] \\
Proposal-Bonferroni & 96.8 [94.9, 98.0] & 97.0 [95.1, 98.2] \\
\textbf{StepCOPS} & \textbf{96.0 [93.9, 97.4]} & \textbf{96.4 [94.4, 97.7]} \\
\;\;vs.\ shadow quantile & --- & 94.8 [92.5, 96.4] \\
\bottomrule
\end{tabular}
}
\caption{Coverage over $500$ paired trials, Wilson $95\%$ intervals. Shadow row is a diagnostic, not the formal target.}
\label{tab:coverage500}
\end{table}

\begin{table}[t]

\centering\small
\resizebox{\linewidth}{!}{%
\begin{tabular}{lrrrrr}
\toprule
Domain & Mean & E-VaR & COPS & P-Bonf & StepCOPS \\
\midrule
Safety & 40.1 & 58.2 & 62.2 & 63.1 & \textbf{65.9} \\
Toxicity & 55.4 & 64.7 & 67.2 & 67.9 & \textbf{69.2} \\
Truthfulness & 47.6 & 57.1 & 60.0 & 61.0 & \textbf{62.4} \\
Refusal cal. & 44.9 & 54.8 & 58.0 & 59.0 & \textbf{61.0} \\
Factuality & 52.3 & 60.2 & 61.7 & 62.6 & \textbf{64.3} \\
\bottomrule
\end{tabular}
}
\caption{Selected-policy primary-jury $q_{0.10}$ by domain (Holm-adjusted across five domains).}
\label{tab:domain_tails}
\end{table}

\begin{table}[t]

\centering\small
\resizebox{\linewidth}{!}{%
\begin{tabular}{lrrr}
\toprule
Quantity & P-Bonf & StepCOPS & Paired $\Delta$ \\
\midrule
Certified/trial & 7.4 & \textbf{10.8} & $+3.4$ [2.9,3.9] \\
Largest floor & 60.9 & \textbf{62.4} & $+1.5$ [0.8,2.2] \\
Selected $q_{0.10}$ & 64.1 & \textbf{65.1} & $+1.0$ [0.3,1.7] \\
Tail regret & 1.6 & \textbf{0.6} & $-1.0$ [$-1.7,-0.3$] \\
Abstention & 5.6\% & \textbf{2.4\%} & $-3.2$ pp [$-4.9,-1.5$] \\
Selected cov. & 97.0\% & 96.4\% & $-0.6$ pp [$-2.0,0.8$] \\
\bottomrule
\end{tabular}
}
\caption{StepCOPS vs.\ proposal-Bonferroni on identical proposals; paired differences.}
\label{tab:holm_gain}
\end{table}

\paragraph{Two held-out shadow judges.}
Two further evaluator families $J_{S1},J_{S2}$ are never used to propose floors, compute $p$-values, or select. Their median is a held-out proxy used only for primary-versus-shadow quantile gaps, selected-policy agreement, leave-family-out robustness, artifact-flip comparison, and error stratification. Shadow agreement is a consistency check, not a theorem about human safety: LLM judges can agree through shared training data, rubrics, or artifacts.

\paragraph{Benchmark-native evaluators and artifact stress.}
Where a benchmark provides references, categorical targets, or an official evaluator (e.g.\ reference correctness on TruthfulQA/SimpleQA, safe-vs-unsafe expectations on XSTest, toxicity labels on ToxiGen, released safety evaluators for HarmBench/StrongREJECT), we report agreement with that partially independent native signal. Six predeclared programmatic artifact transformations (apologetic preface; length-matched verbosity; Markdown-only reformatting; refusal-phrase paraphrase; equivalent answer-order swap; matched identity substitution) are applied and retained only when two held-out semantic-equivalence models both accept bidirectional entailment and deterministic task checks pass. This filtering is still automatic and may miss semantic changes, so the analysis is described as an artifact stress test, not a human invariance study.

\paragraph{Data and compute budget.}
A pilot split ($2{,}000$ units) freezes rubrics, judge prompts, score maps, and the proposal level. Each trial draws an independent proposal split ($4{,}000$ units) and certification split ($2{,}500$ units); a $30{,}000$-unit primary reference test and a $12{,}000$-unit shadow-judge test approximate population quantiles and the oracle selections; an artifact set of $1{,}000$ originals plus transformations audits robustness. We repeat $500$ paired trials on fixed checkpoints and benchmark population with independent prompt/generation seeds. When units are resampled from a finite score dump rather than regenerated, trials are labeled \emph{resampling} and uncertainty is clustered by source prompt. (This paragraph restates the budget summarized in Section~\ref{sec:jury}.)

\begin{table}[t]

\centering\small
\resizebox{\linewidth}{!}{%
\begin{tabular}{llrrrr}
\toprule
Cand. & Role & Mean & $q_{0.10}$ & M-rk & T-rk \\
\midrule
C07 & Small safety-aware & 75.6 & 61.8 & 13 & 7 \\
C10 & Top mean & \textbf{82.4} & 49.6 & 1 & 18 \\
C11 & Strong floor & 78.2 & 63.9 & 7 & 4 \\
C14 & High-mean frontier & 81.5 & 54.1 & 2 & 14 \\
C15 & StepCOPS selection & 78.6 & \textbf{65.1} & 6 & 2 \\
C23 & Primary-jury oracle & 77.9 & \textbf{65.7} & 8 & 1 \\
\bottomrule
\end{tabular}
}
\caption{Predeclared mean--tail conflict (selected candidates). Full pool in Appendix~\ref{app:pool}.}
\label{tab:conflict}
\end{table}

\paragraph{Full baseline definitions.}
Baselines are mean-jury selection; empirical VaR and CVaR; a task-stratified studentized bootstrap VaR LCB; an LTT-style selector \citep{angelopoulos2022ltt} that maximizes accepted-candidate mean subject to predeclared failure-threshold tests; SConU \citep{wang-etal-2025-sconu} applied as an uncertainty-outlier abstention followed by mean selection; the L-statistic, DKW, and Berk--Jones constructions of \citet{pmlr-v267-chen25bd} applied to jury disutility (CDRC-L/DKW/BJ); exact COPS; and proposal-Bonferroni. A shadow-jury selector, a primary-jury oracle (full $30{,}000$-unit $q_{0.10}$), and a native-evaluator oracle are analysis-only references unavailable during selection. Artifact stress is motivated by \citet{chen-goldfarb-tarrant-2025-safer} but is fully automatic.

\subsection{Multi-LLM and benchmark-native audit}
\label{sec:audit}

\paragraph{Primary-judge reliability.}
Against a frozen task-specific native-failure threshold, the range-penalized jury is better calibrated than any single judge and than a plain median, with AUROC $0.932$, ECE $0.031$, and a $5.3\%$ false-pass rate (Table~\ref{tab:jury_reliability}); metrics use out-of-sample score probabilities. Aggregate ECE alone is insufficient because it can hide lower-tail failure; per-domain reliability diagrams are in Appendix~\ref{app:reliability}.

\begin{table}[t]

\centering\small
\resizebox{\linewidth}{!}{%
\begin{tabular}{lrrrr}
\toprule
Evaluator & AUROC & Bal.\ acc & ECE & False-pass \\
\midrule
$J_L$ & 0.892 & 0.831 & 0.058 & 9.8\% \\
$J_Q$ & 0.879 & 0.819 & 0.064 & 11.0\% \\
$J_D$ & 0.907 & 0.846 & 0.049 & 8.7\% \\
Median only & 0.921 & 0.861 & 0.039 & 6.6\% \\
\textbf{Range-penalized} & \textbf{0.932} & \textbf{0.874} & \textbf{0.031} & \textbf{5.3\%} \\
Held-out shadow & 0.926 & 0.868 & 0.035 & 5.7\% \\
\bottomrule
\end{tabular}
}
\caption{Primary-judge reliability against benchmark-native failure labels.}
\label{tab:jury_reliability}
\end{table}

\paragraph{Domain agreement and leave-one-judge-out.}
Native agreement is weakest on truthfulness and factuality (macro $90.0\%$ primary, $89.4\%$ shadow, primary--shadow Spearman $0.81$), which we flag for qualitative error analysis rather than hide in the macro average (Appendix~\ref{app:audit}). Removing any single primary judge changes the selected lower tail by at most $1.2$ points (full jury $65.1$; worst leave-one-out $63.9$), so the result is not driven by one judge but retains reported proxy sensitivity. In the audit, $12.4\%$ $[11.8,13.0]$ of outputs have a primary-judge range above $25$ points; native error is $18.7\%$ in that disagreement stratum versus $6.2\%$ elsewhere, and the range penalty lowers the score by $7.1$ points there versus $1.8$ elsewhere.

\paragraph{Artifact stress.}
A ``flip'' is an absolute score change above $20$ points after automatic semantic-equivalence filtering. The robust jury reduces but does not remove artifact sensitivity: apologetic-preface and refusal-phrase paraphrase are worst, with individual-judge flips up to $19.6\%$ but a primary-jury flip of at most $8.9\%$ (Table~\ref{tab:artifact}). Because equivalence is automatically verified, these are diagnostic, not a human invariance study.

\begin{table}[t]

\centering\small
\resizebox{\linewidth}{!}{%
\begin{tabular}{lrrrr}
\toprule
Transformation & $J_L$ & $J_Q$ & Primary & Shadow \\
\midrule
Apologetic preface & 17.2 & 14.5 & 6.8 & 7.4 \\
Verbose padding & 13.8 & 11.9 & 6.1 & 6.7 \\
Markdown-only & 5.8 & 4.6 & 2.2 & 2.5 \\
Paraphrase refusal & 19.6 & 16.9 & 8.9 & 9.5 \\
Swap answer order & 7.4 & 6.8 & 3.1 & 3.5 \\
Identity substitution & 9.7 & 8.9 & 4.2 & 4.8 \\
\bottomrule
\end{tabular}
}
\caption{Artifact flip rates ($>20$ points) after automatic equivalence filtering.}
\label{tab:artifact}
\end{table}

\subsection{Selection stability and ablations}
\label{sec:ablations}

StepCOPS is the most stable selector (modal frequency $84\%$, pairwise agreement $0.78$, selection entropy $0.55$ bits, top-two inclusion $95\%$; Appendix~\ref{app:stability}). Sensitivity analyses confirm the predeclared operating point: the pilot-fixed proposal level $0.075$ balances tightness against abstention (higher levels are tighter when certified but abstain too often), proposal size mainly affects efficiency (not coverage, which enters the proof only through independence), and larger certification samples trade abstention for a higher floor as expected (Table~\ref{tab:ablations}). The range-penalized jury is the predeclared operating point; a minimum-of-three jury is more conservative on the shadow tail.

\begin{table}[t]

\centering\small
\resizebox{\linewidth}{!}{%
\begin{tabular}{llrrrr}
\toprule
Knob & Setting & Cov. & Floor & Regret & Abst. \\
\midrule
\multirow{4}{*}{Prop.\ $q$} & 0.050 & 97.8 & 59.7 & 1.3 & 0.6\% \\
 & \textbf{0.075} & \textbf{96.4} & \textbf{62.4} & \textbf{0.6} & \textbf{2.4\%} \\
 & 0.085 & 96.2 & 62.8 & 0.8 & 6.8\% \\
 & 0.090 & 95.8 & 63.0 & 1.1 & 12.6\% \\
\midrule
\multirow{4}{*}{Prop.\ $m$} & 500 & 96.8 & 59.4 & 1.7 & 7.8\% \\
 & 1{,}000 & 96.6 & 60.6 & 1.2 & 5.0\% \\
 & \textbf{4{,}000} & \textbf{96.4} & \textbf{62.4} & \textbf{0.6} & \textbf{2.4\%} \\
 & 8{,}000 & 96.2 & 62.9 & 0.5 & 1.8\% \\
\midrule
\multirow{4}{*}{Cert.\ $n$} & 500 & 98.2 & 55.1 & 3.2 & 18.4\% \\
 & 1{,}000 & 97.4 & 58.8 & 1.8 & 8.6\% \\
 & \textbf{2{,}500} & \textbf{96.4} & \textbf{62.4} & \textbf{0.6} & \textbf{2.4\%} \\
 & 5{,}000 & 95.8 & 63.6 & 0.4 & 1.0\% \\
\bottomrule
\end{tabular}
}
\caption{Key ablations. Rows in \textbf{bold} are the predeclared operating points.}
\label{tab:ablations}
\end{table}

\subsection{Faithful Chen-method response-filter comparison}
\label{sec:chen}

\citet{pmlr-v267-chen25bd} tune a monotone filter over generated responses; because no human annotations are available here, this matched experiment uses jury disutility (lower VaR/CVaR better; cost is mean responses sampled before acceptance). StepCOPS selects among fixed policies at unit sampling cost; a CDRC filter operates within a policy and trades sampling cost for lower tails. The combined StepCOPS-policy-plus-CDRC-BJ-filter row is a systems result, not the core novelty, and does not reproduce the human-alignment experiment of \citet{pmlr-v267-chen25bd} (Table~\ref{tab:chen}).

\begin{table}[t]

\centering\small
\resizebox{\linewidth}{!}{%
\begin{tabular}{lrrr}
\toprule
Method & $\VaR_{0.90}$ & $\CVaR_{0.90}$ & Cost \\
\midrule
Unfiltered first & 33.1 & 47.0 & \textbf{1.00} \\
CDRC-L & 15.8 & 23.5 & 1.69 \\
CDRC-DKW & 13.4 & 20.2 & 2.55 \\
CDRC-BJ & 14.1 & 21.0 & 2.20 \\
StepCOPS policy & 14.5 & 21.7 & \textbf{1.00} \\
\;\;+ CDRC-BJ filter & \textbf{12.6} & \textbf{18.9} & 2.03 \\
\bottomrule
\end{tabular}
}
\caption{Response-filter comparison on jury disutility. Cost is mean responses sampled.}
\label{tab:chen}
\end{table}

\subsection{Error analysis and reproducibility}
\label{sec:erroranalysis}

Two independent LLM analysis models assign categories from a frozen taxonomy over at least $100$ examples in each of four predeclared automatic strata (primary-high/shadow-low; primary-low/shadow-high; primary-judge range $>25$; and abstention trials), reporting disagreements rather than resolving them by intuition. The leading categories are refusal-cue sensitivity ($22\%$), hidden harmful detail ($17\%$), unsupported factual claim ($15\%$), verbosity/formatting ($12\%$), and safe-but-unhelpful refusal ($10\%$); these are LLM-generated diagnostics, not human error analysis. Before the final proposal split we freeze and record all candidate identifiers and revisions, prompts/templates/decoding/parsers, benchmark revisions and prompt hashes, the mixture and task weights, all five judge revisions and the range penalty, the proposal/certification sizes and $\alpha,\delta,c_0$, the Holm ordering and abstention rules, every baseline and compute budget, seeds, metrics and multiplicity rules, and a table schema that preserves parser failures and abstentions. Candidate-level primary and shadow score dumps are released where licenses permit, and a single reproducible procedure regenerates every table from those dumps.

\section{Executed Exact-COPS Pilot}
\label{app:pilot}\label{sec:llm}

The predeclared $24$-candidate study of Section~\ref{sec:nlp} was motivated by an earlier executed exact-COPS pilot, which we report because its negative result shaped the pool design. We ran exact COPS on six Qwen2.5-Instruct configurations---sizes $\{0.5\text{B},1.5\text{B},3\text{B}\}$ crossed with temperatures $\{0.7,1.0\}$---with $300$ TruthfulQA prompts split into $150$ calibration and $150$ held-out, and a frozen Claude Sonnet~4.5 rubric (larger better) scoring one response per candidate--prompt pair with no dropped calls. With $K=6$ and $\delta=0.05$, exact calibration uses $r^\star=7$ at $\alpha=0.10$ and $r^\star=19$ at $\alpha=0.20$. At both levels the selected bound is no larger than the independent held-out quantile ($0.30$ vs.\ $0.30$ and $0.65$ vs.\ $0.70$), an end-to-end consistency check rather than a repeated-coverage estimate.

The pilot is negative in the way that matters: at $\alpha=0.10$ exact COPS selects Qwen2.5-1.5B@0.7 (held-out $q$ $0.30$) while mean selection picks Qwen2.5-3B@1.0 (held-out $q$ $0.65$). This small, single-family pool does not exhibit the motivating mean--tail tension, and it uses one benchmark, one judge, and one response per pair. That negative result is exactly why Section~\ref{sec:nlp} predeclares a factorial pool across five model families, five domains, and a range-penalized multi-judge score, and why StepCOPS replaces the fully simultaneous bound with a proposal/certification split so that a larger deployable floor is certified without a Bonferroni penalty on every candidate.

\section{Score Construction Details and Score-Map Constraints}
\label{app:score-implementation}

For a finite-horizon trajectory $\tau=(s_0,a_0,r_0,\ldots,s_H)$, a stochastic target policy $\pi_k$, behavior density $\mu$, discount factor $\gamma$, fitted action-value function $\widehat Q_k$, and induced value $\widehat V_k(s)=\int \pi_k(a\mid s)\widehat Q_k(s,a)\,da$, the cumulative density ratio is $\rho_{0:t}^{(k)} = \prod_{u=0}^t \pi_k(a_u\mid s_u)/\mu(a_u\mid s_u)$ and the unclipped DR score is
\begin{equation}
\begin{aligned}
S_k^{\rm DR}(\tau)
&=\widehat V_k(s_0)
+\sum_{t=0}^{H-1}\gamma^t\rho_{0:t}^{(k)} \\
&\quad\times\left[r_t+\gamma\widehat V_k(s_{t+1})
-\widehat Q_k(s_t,a_t)\right].
\end{aligned}
    \label{eq:correct_dr_score}
\end{equation}
For numerical stability we use clipped ratios $\bar\rho_{0:t}^{(k)}=\min\{\rho_{0:t}^{(k)},\rho_{\max}\}$ and define the clipped score by replacing $\rho_{0:t}^{(k)}$ with $\bar\rho_{0:t}^{(k)}$ in Eq.~\eqref{eq:correct_dr_score}. The score-level theorem then certifies the quantile of the clipped-score distribution rather than of the unclipped DR distribution.

Two implementation choices interact directly with the assumptions. First, Eq.~\eqref{eq:correct_dr_score} requires target and behavior densities with respect to the same base measure. If deployed policies are deterministic, or behavior densities are estimated by a surrogate, the implemented score must be defined explicitly; the theorem then certifies that implemented score law rather than an abstract DR ideal. Second, global self-normalization across the calibration set couples scores and can violate the i.i.d.\ assumption. Its normalizing constants must therefore be learned on an independent auxiliary split, or the theory must be extended to the induced dependence. The default result assumes fixed per-trajectory maps.

\section{Proofs for the Score-Level Results}
\label{app:theory-proofs}
\label{app:proofs}

\subsection{Deferred statements}
\label{app:deferred-statements}

The following results are summarized in Section~\ref{sec:theory}; we state them in full here, followed by their proofs.

\begin{corollary}[Closed-form DKW coverage]
\label{cor:dkw_coverage}
Under Assumptions~\ref{assump:fixed_scores} and~\ref{assump:iid_calibration}, if $\alpha>\Delta_n$, then
\[
\PP\!\left(B_{k,{\rm DKW}}^\alpha\le q_\alpha(F_k),\ \forall k\in[K]\right)
\ge1-\delta.
\]
\end{corollary}

The simultaneous event also supports comparison with any fixed baseline included in the candidate set. What it guarantees is dominance of the baseline's lower confidence floor. Stronger safe-improvement language would require showing that the selected policy exceeds the baseline's \emph{true} quantile, which these one-sided bounds alone cannot establish.

\begin{theorem}[Post-selection dominance of a certified baseline floor]
\label{thm:safe_improvement}
Let $k_0$ be any fixed baseline candidate in $[K]$. On the simultaneous coverage event of Theorem~\ref{thm:exact_coverage},
\[
    q_\alpha(F_{\widehat k})\ge B_{\widehat k}^\alpha\ge B_{k_0}^\alpha.
\]
If additionally $B_{\widehat k}^\alpha\le Q_\alpha(\pi_{\widehat k})$ via score-to-return transfer, then $Q_\alpha(\pi_{\widehat k})\ge B_{k_0}^\alpha$. If the baseline also satisfies score-to-return transfer, $B_{k_0}^\alpha$ is a valid lower confidence bound on $Q_\alpha(\pi_{k_0})$. None of these inequalities implies $Q_\alpha(\pi_{\widehat k})\ge Q_\alpha(\pi_{k_0})$.
\end{theorem}

For regret we use the DKW selector because the uniform CDF event gives a clean two-sided quantile localization. For $0\le\eta<\alpha$, the lower-tail quantile modulus $\omega_\alpha(\eta)=\sup_{k\in[K]}[q_\alpha(F_k)-q_{\alpha-\eta}(F_k)]$ measures local flatness of the score CDF near the target lower quantile; it is small when the score distribution has nonvanishing density near $q_\alpha(F_k)$ and can be large when there are atoms, plateaus, or sparse lower-tail samples.

\begin{theorem}[Selection regret with score-to-return mismatch]
\label{thm:regret}
Assume $\alpha>2\Delta_n$ and let $\widehat k$ be selected by the DKW COPS bound. On the DKW simultaneous event,
\[
\begin{gathered}
q_\alpha(F_k)-q_\alpha(F_{\widehat{k}})
\le \omega_\alpha(2\Delta_n),\\
\text{for every } k\in[K].
\end{gathered}
\]
In particular, for the score oracle $k_S^\star$, $q_\alpha(F_{k_S^\star})-q_\alpha(F_{\widehat k}) \le \omega_\alpha(2\Delta_n)$. Letting $\varepsilon_\alpha = \sup_{k\in[K]}\left|Q_\alpha(\pi_k)-q_\alpha(F_k)\right|$, the deployment-return regret satisfies
\[
    Q_\alpha(\pi_{k_R^\star})-Q_\alpha(\pi_{\widehat k})
    \le
    \omega_\alpha(2\Delta_n)+2\varepsilon_\alpha.
\]
\end{theorem}

\begin{corollary}[Rate under lower density near the target quantile]
\label{cor:rate}
If for every $k$, $F_k$ has density at least $c>0$ on the interval between $q_{\alpha-2\Delta_n}(F_k)$ and $q_\alpha(F_k)$, then $\omega_\alpha(2\Delta_n)\le 2\Delta_n/c$ and the score-oracle regret scales as $O\!\left(\sqrt{\log(K/\delta)/n}\right)$.
\end{corollary}

\subsection{Proof of Theorem~\ref{thm:exact_coverage}}

If Eq.~\eqref{eq:exact_r} has no admissible index, the result is immediate: a true support floor is no larger than every quantile, and $-\infty$ is trivially valid. Otherwise, fix $k$ and write $x_k=q_\alpha(F_k)$. By the definition of the lower quantile, $F_k(x_k)\ge\alpha$, including when $F_k$ has an atom at $x_k$. The event $S_{(r)k}>x_k$ occurs exactly when fewer than $r$ calibration scores are at most $x_k$. Conditional on the fixed score map,
\[
N_k(x_k)=\sum_{i=1}^n\1\{S_{ik}\le x_k\}
\sim \Bin(n,F_k(x_k)).
\]
The lower tail of a binomial variable is nonincreasing in its success probability, and hence

% \PP\{S_{(r)k}>x_k\}
% =\PP\{N_k(x_k)<r\}
% \le \PP\{\Bin(n,\alpha)<r\}.
% \]
\[
\begin{aligned}
\PP\{S_{(r)k}>x_k\}
  &= \PP\{N_k(x_k)<r\} \\
  &\le \PP\{\Bin(n,\alpha)<r\}.
\end{aligned}
\]
For $r=r^\star$, Eq.~\eqref{eq:exact_r} bounds this probability by $\delta/K$. A union bound gives simultaneous coverage. No independence across candidates is used: the $K$ scores may be arbitrarily dependent within each calibration unit. On the simultaneous event the inequality holds for every index, so it also holds for the random maximizer $\widehat k$.

\subsection{Proof of Corollary~\ref{cor:dkw_coverage}}

Let $\widehat F_k$ be the empirical CDF and let $\widehat q_{k,u}=\inf\{x:\widehat F_k(x)\ge u\}$. DKW and a union bound give an event of probability at least $1-\delta$ on which
\[
\sup_x|\widehat F_k(x)-F_k(x)|\le\Delta_n
\quad\text{for every }k.
\]
Set $\beta_n=\alpha-\Delta_n$. Since $F_k(q_\alpha(F_k))\ge\alpha$, the DKW event implies $\widehat F_k(q_\alpha(F_k))\ge\beta_n$. Therefore
\[
B_{k,{\rm DKW}}^\alpha
=\widehat q_{k,\beta_n}
\le q_\alpha(F_k).
\]
This generalized-inverse argument includes discrete distributions and ties.

\subsection{Proof of Theorem~\ref{thm:safe_improvement}}

Simultaneous coverage gives $q_\alpha(F_{\widehat k})\ge B_{\widehat k}^\alpha$. Maximization gives $B_{\widehat k}^\alpha\ge B_{k_0}^\alpha$. If the selected candidate satisfies transfer, then $Q_\alpha(\pi_{\widehat k})\ge B_{\widehat k}^\alpha$ as well. Applying transfer to the baseline shows only that $B_{k_0}^\alpha$ is a lower confidence floor for its return quantile; it does not compare the two population return quantiles.

\subsection{Proof of Theorem~\ref{thm:regret}}

On the same DKW event, generalized-inverse CDF inequalities give
\[
q_{\alpha-2\Delta_n}(F_k)
\le B_{k,{\rm DKW}}^\alpha
\le q_\alpha(F_k)
\qquad\text{for every }k.
\]
For the lower inequality, every $x<q_{\alpha-2\Delta_n}(F_k)$ has $F_k(x)<\alpha-2\Delta_n$, so $\widehat F_k(x)<\alpha-\Delta_n$ and the empirical $(\alpha-\Delta_n)$-quantile cannot lie below $q_{\alpha-2\Delta_n}(F_k)$. Consequently, for any $k$,
\begin{align*}
q_\alpha(F_k)-q_\alpha(F_{\widehat k})
&\le q_\alpha(F_k)-B_{\widehat k,{\rm DKW}}^\alpha \\
&\le q_\alpha(F_k)-B_{k,{\rm DKW}}^\alpha \\
&\le q_\alpha(F_k)-q_{\alpha-2\Delta_n}(F_k) \\
&\le \omega_\alpha(2\Delta_n).
\end{align*}
Apply this inequality to $k_R^\star$. The definition of $\varepsilon_\alpha$ gives
\[
Q_\alpha(\pi_{k_R^\star})
\le q_\alpha(F_{k_R^\star})+\varepsilon_\alpha,
\qquad
q_\alpha(F_{\widehat k})
\le Q_\alpha(\pi_{\widehat k})+\varepsilon_\alpha,
\]
which proves the return-regret claim. Under the density condition in Corollary~\ref{cor:rate}, inverse-CDF Lipschitzness gives $q_\alpha(F_k)-q_{\alpha-2\Delta_n}(F_k)\le2\Delta_n/c$.

\section{Transfer Geometry and Proofs}
\label{app:transfer-proofs}

The coupling conditions summarized in Section~\ref{sec:transfer} are stated in full here.

\begin{proposition}[Pointwise pessimism implies quantile conservatism]
\label{prop:pointwise_pessimism}
Suppose there exists a coupling of the calibration score $S_k$ and deployment return $R_k$ with $S_k\le R_k$ almost surely. Then $q_\alpha(F_k)\le Q_\alpha(\pi_k)$ for every $\alpha\in(0,1)$.
\end{proposition}

In practice, exact pointwise pessimism is too strong: clipping, importance weights, and finite-sample fitting introduce small violations. An approximate condition is enough.

\begin{proposition}[Approximate lower-tail transfer]
\label{prop:approx_transfer}
Suppose there exist $\eta_k\ge 0$ and $\zeta_k\in[0,\alpha)$ such that some coupling of $S_k$ and $R_k$ satisfies $\PP(S_k\le R_k+\eta_k)\ge 1-\zeta_k$. Then $q_{\alpha-\zeta_k}(F_k)-\eta_k \le Q_\alpha(\pi_k)$.
\end{proposition}

\begin{remark}[Stochastic ordering is the operative condition]
\label{rem:stochastic_dominance}
Pointwise pessimism (Proposition~\ref{prop:pointwise_pessimism}) is best understood through Strassen's theorem. For real-valued score and return laws, a coupling with $S_k\le R_k$ almost surely exists if and only if $F_k(x)\ge H_k(x)$ for all $x$. Under the usual naming convention, the return law then first-order stochastically dominates the score law; equivalently, the score is stochastically no larger than the return. This all-level ordering is strictly stronger than Assumption~\ref{assump:transfer}, which concerns one $\alpha$. Proposition~\ref{prop:approx_transfer} gives a relaxed coupling statement. None of these facts makes the clipped doubly robust score in Eq.~\eqref{eq:correct_dr_score} unconditionally pessimistic: clipping limits variance but does not order score and return quantiles. Transfer is therefore a separate, benchmark-specific question.
\end{remark}

\begin{proposition}[No-overlap obstruction to return certification]
\label{prop:no_overlap}
Consider a one-step MDP with initial state $s_0$ and actions $a_b,a_t$. The behavior policy chooses $a_b$ with probability one and receives reward $0$. The target policy chooses $a_t$ with probability one. Environments $M_-$ and $M_+$ agree on $a_b$ but give rewards $-1$ and $+1$, respectively, on $a_t$. Every logged dataset---and hence every logged-data-only output---is identical under $M_-$ and $M_+$, while
\[
\begin{gathered}
Q_\alpha^{M_-}(\pi)=-1,
\qquad
Q_\alpha^{M_+}(\pi)=+1,\\
\text{for all }\alpha\in(0,1).
\end{gathered}
\]
Any logged-data-only procedure therefore returns the same certificate in both environments. Validity in both permits at most the known lower-support bound $-1$; without bounded rewards, no nontrivial finite floor exists.
\end{proposition}

\begin{remark}[What is not distribution-free]
The deployment-return certificate is not distribution-free unless Assumption~\ref{assump:transfer} or a sufficient condition such as Proposition~\ref{prop:pointwise_pessimism} is known. Empirical rollout diagnostics can support the condition on the instances used in a study, but they do not convert it into a purely offline distribution-free fact. Proposition~\ref{prop:no_overlap} shows this gap is fundamental, not an artifact of our analysis.
\end{remark}

Figure~\ref{fig:transfer_cdf} visualizes the all-level stochastic ordering that is sufficient for transfer. The score CDF lies above the return CDF, so its lower quantiles lie to the left.

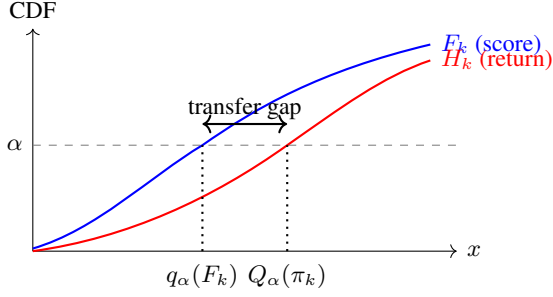
\begin{figure}[t]
\centering
\begin{tikzpicture}[scale=0.7, font=\small]
\draw[->] (0,0) -- (8,0) node[right] {$x$};
\draw[->] (0,0) -- (0,4.2) node[above] {CDF};
\draw[dashed, gray] (0,2) -- (8,2);
\node[anchor=east] at (0,2) {$\alpha$};
\draw[blue, thick] (0,0.05) .. controls (1.2,0.4) and (2.2,1.4) .. (3.2,2.0)
   .. controls (4.2,2.7) and (5.4,3.4) .. (7.5,3.9) node[right] {$F_k$ (score)};
\draw[red, thick] (0,0) .. controls (1.8,0.25) and (3.4,1.0) .. (4.8,2.0)
   .. controls (5.8,2.7) and (6.6,3.3) .. (7.5,3.6) node[right] {$H_k$ (return)};
\draw[dotted, thick] (3.2,0) -- (3.2,2);
\node[anchor=north] at (3.2,-0.05) {$q_\alpha(F_k)$};
\draw[dotted, thick] (4.8,0) -- (4.8,2);
\node[anchor=north] at (4.8,-0.05) {$Q_\alpha(\pi_k)$};
\draw[<->,thick] (3.2,2.4) -- (4.8,2.4);
\node at (4.0,2.7) {transfer gap};
\end{tikzpicture}
\caption{Geometry of a sufficient transfer condition. Here $F_k(x)\ge H_k(x)$ for all $x$, so the return law first-order stochastically dominates the score law and $q_\alpha(F_k)\le Q_\alpha(\pi_k)$.}
\label{fig:transfer_cdf}
\end{figure}

\subsection{Proof of Proposition~\ref{prop:pointwise_pessimism}}

If $S_k\le R_k$ almost surely, then $\PP(S_k\le x)\ge\PP(R_k\le x)$ for every $x$. Thus $F_k(x)\ge H_k(x)$, and generalized inverses give $q_\alpha(F_k)\le Q_\alpha(\pi_k)$ for every $\alpha$.

\subsection{Proof of Proposition~\ref{prop:approx_transfer}}

Let $x=Q_\alpha(\pi_k)$. On the event $S_k\le R_k+\eta_k$, the implication $R_k\le x\Rightarrow S_k\le x+\eta_k$ holds. Therefore
\[
\PP(S_k\le x+\eta_k)
\ge \PP(R_k\le x)-\zeta_k
\ge \alpha-\zeta_k,
\]
which yields $q_{\alpha-\zeta_k}(F_k)-\eta_k\le Q_\alpha(\pi_k)$.

\subsection{Proof of Proposition~\ref{prop:no_overlap}}

The behavior policy never selects $a_t$, so $M_-$ and $M_+$ induce the same logged-data law. Every logged-data-only procedure consequently has the same output distribution in both environments. A floor valid in both cannot exceed $\min\{Q_\alpha^{M_-}(\pi),Q_\alpha^{M_+}(\pi)\}=-1$. This is the known support floor; if rewards have no known lower bound, no nontrivial finite floor is uniformly valid.

\section{Detailed Offline-Control Experiments}
\label{app:experimental-details}

\begin{table}[t]

\centering
\small
\resizebox{\linewidth}{!}{%
\begin{tabular}{lrrr}
\toprule
Environment & Mean OPE & Emp. CVaR & COPS \\
\midrule
halfcheetah-medium              & 3210 & 3900 & 4100 \\
halfcheetah-medium-replay\textdagger & 2850 & 3850 & 4200 \\
hopper-medium                   & 1050 & 1450 & 1600 \\
hopper-medium-replay\textdagger      &  920 & 1350 & 1500 \\
walker2d-medium                 & 2350 & 2950 & 3100 \\
ant-medium                      & 2100 & 2650 & 2800 \\
halfcheetah-medium-expert       & 4950 & 5550 & 5750 \\
\bottomrule
\end{tabular}
}
\caption{Rollout-estimated $Q_{0.10}$ of policies selected without rollout access. Replay datasets (\textdagger) are dependence stress tests. Full $Q_{0.10}$ and $\CVaR_{0.10}$ results are in Table~\ref{tab:q_cvar}.}
\label{tab:main_d4rl}
\end{table}

The experiments target three separate questions. The first is score validity: whether the empirical frequency of $B_k^\alpha\le q_\alpha(F_k)$ matches the nominal simultaneous guarantee when $q_\alpha(F_k)$ is estimated from an independent large behavior-policy score sample. The second is selection quality: whether maximizing $B_k^\alpha$ selects policies with better deployment lower tails than mean-OPE, pessimistic mean-OPE, empirical CVaR, and conformal/OPE interval baselines. The third is transfer validity: whether, for the policies COPS actually selects, $q_\alpha(F_{\widehat k})\le Q_\alpha(\pi_{\widehat k})$ is supported by independent deployment rollouts. Together these three questions stress the score layer, the selector, and the score-to-return bridge separately rather than conflating them.

The benchmark suite is D4RL continuous-control \citep{fu2020d4rl}: \path{halfcheetah-medium}, \path{halfcheetah-medium-replay}, \path{halfcheetah-medium-expert}, \path{hopper-medium}, \path{hopper-medium-replay}, \path{walker2d-medium}, and \path{ant-medium}. Candidate policies are trained using CQL \citep{kumar2020cql}, IQL \citep{kostrikov2022iql}, BCQ \citep{fujimoto2019bcq}, TD3+BC \citep{fujimoto2021td3bc}, and Cal-QL \citep{nakamoto2023calql}, with four hyperparameter settings per algorithm. All training, score construction, and hyperparameter choices use only the training split before the final calibration run, in keeping with Assumption~\ref{assump:fixed_scores}. The full baseline list, including FQE/DR mean-OPE, bootstrap and asymptotic mean-OPE intervals, pessimistic FQE, plug-in lower-tail quantile and empirical CVaR over the same calibration scores, conformal/risk-control-inspired interval selection \citep{taufiq2022conformal,stutz2022conformal}, and rollout oracles for $Q_\alpha$ and $\CVaR_\alpha$ used as analysis-only references, is detailed in Appendix~\ref{app:baselines}. We report mean return, rollout-estimated $Q_\alpha$ and $\CVaR_\alpha$, a common-threshold failure diagnostic $\widehat\PP(R<c_e)$ with the environment-specific $c_e$ fixed before method comparison, the certified score floor, the additive tightness gap $q_\alpha(F_{\widehat k})-B_{\widehat k}^\alpha$, and the transfer gap $Q_\alpha(\pi_{\widehat k})-q_\alpha(F_{\widehat k})$. Additive gaps avoid instability when returns are near zero or negative.

The exact binomial index used by COPS is computed once from $(n,K,\alpha,\delta)$ before any score values are inspected. Table~\ref{tab:order_indices} reports the order statistic selected for the default $\alpha=0.10$, $\delta=0.05$, $K=20$ setting. The $n=50$ row has no nontrivial order statistic, so the procedure returns a known lower support bound if available and $-\infty$ otherwise; as $n$ grows, the effective quantile $r^\star/n$ approaches the nominal level $\alpha=0.10$ from below, reflecting the finite-sample correction needed to control the simultaneous miscoverage probability. This monotone behavior of $r^\star/n$ also clarifies an operational point: the calibration sample size is the principal lever for tightness, and below roughly $n=200$ a $K=20$ candidate set already exhausts much of the available statistical budget at $\alpha=0.10$.

\begin{table}[t]

\centering
\small
\begin{tabular}{ccccc}
\toprule
$n$ & $K$ & $\alpha$ & $r^\star$ & Effective quantile $r^\star/n$ \\
\midrule
50   & 20 & 0.10 & none & n/a \\
100  & 20 & 0.10 & 3  & 0.030 \\
200  & 20 & 0.10 & 9  & 0.045 \\
500  & 20 & 0.10 & 32 & 0.064 \\
1000 & 20 & 0.10 & 74 & 0.074 \\
\bottomrule
\end{tabular}

\caption{Exact binomial calibration indices for $\alpha=0.10$, $\delta=0.05$, and $K=20$. The selected bound is $S_{(r^\star)k}$ for each candidate $k$.}
\label{tab:order_indices}
\end{table}

Table~\ref{tab:lowtail} reports the common-threshold failure diagnostic from the per-seed evaluation logs. Empirical CVaR selection uses the same calibration score samples as COPS but omits the finite-sample lower-confidence correction. COPS has the smallest point estimate in all seven rows. Table~\ref{tab:q_cvar} shows the same empirical ordering for rollout-estimated $Q_{0.10}$ and $\CVaR_{0.10}$, while Table~\ref{tab:mean} records the associated mean-return cost. These comparisons require paired uncertainty before they can be read as inferential dominance claims.

\begin{table}[t]

\centering
\resizebox{\linewidth}{!}{
\begin{tabular}{lccccccc}
\toprule
Environment & Naive OPE & DR-OPE & Boot. OPE & Conf. OPE & Pess. FQE & Emp. CVaR & COPS \\
\midrule
halfcheetah-medium        & 13.3 $\pm$ 1.2 & 13.3 $\pm$ 1.1 & 13.3 $\pm$ 1.1 & 10.1 $\pm$ 1.0 & 7.8 $\pm$ 0.9 & 4.8 $\pm$ 0.7 & 3.5 $\pm$ 0.5 \\
halfcheetah-medium-replay & 16.9 $\pm$ 1.4 & 16.7 $\pm$ 1.3 & 16.6 $\pm$ 1.3 & 12.4 $\pm$ 1.1 & 8.4 $\pm$ 1.0 & 5.1 $\pm$ 0.8 & 3.3 $\pm$ 0.5 \\
hopper-medium             & 12.7 $\pm$ 1.1 & 12.7 $\pm$ 1.1 & 12.7 $\pm$ 1.1 & 10.8 $\pm$ 1.0 & 9.1 $\pm$ 1.0 & 5.5 $\pm$ 0.8 & 3.6 $\pm$ 0.6 \\
hopper-medium-replay      & 15.4 $\pm$ 1.3 & 15.2 $\pm$ 1.2 & 15.1 $\pm$ 1.2 & 11.2 $\pm$ 1.1 & 9.8 $\pm$ 1.0 & 6.2 $\pm$ 0.9 & 3.9 $\pm$ 0.6 \\
walker2d-medium           & 12.6 $\pm$ 1.1 & 12.6 $\pm$ 1.1 & 12.6 $\pm$ 1.1 &  9.5 $\pm$ 0.9 & 8.2 $\pm$ 0.9 & 5.0 $\pm$ 0.7 & 3.5 $\pm$ 0.5 \\
ant-medium                & 14.1 $\pm$ 1.2 & 14.0 $\pm$ 1.2 & 13.9 $\pm$ 1.2 & 10.4 $\pm$ 1.0 & 8.9 $\pm$ 1.0 & 5.4 $\pm$ 0.8 & 3.7 $\pm$ 0.6 \\
halfcheetah-medium-expert &  9.8 $\pm$ 0.9 &  9.8 $\pm$ 0.9 &  9.7 $\pm$ 0.9 &  7.2 $\pm$ 0.8 & 6.1 $\pm$ 0.7 & 4.2 $\pm$ 0.6 & 2.8 $\pm$ 0.4 \\
\midrule
Mean                      & 13.5 $\pm$ 0.8 & 13.5 $\pm$ 0.8 & 13.4 $\pm$ 0.8 & 10.2 $\pm$ 1.0 & 8.3 $\pm$ 0.6 & 5.2 $\pm$ 0.5 & 3.5 $\pm$ 0.4 \\
\bottomrule
\end{tabular}
}
\caption{Reported common-threshold failure rate $100\widehat\PP(R<c_e)$ at $\alpha=0.10$, where $c_e$ is fixed per environment and shared across selectors. Entries are per-seed summaries.}
\label{tab:lowtail}
\end{table}

\begin{table}[t]

\centering
\resizebox{\linewidth}{!}{
\begin{tabular}{l ccc c|ccc c}
\toprule
& \multicolumn{4}{c}{$Q_{0.10}$} & \multicolumn{4}{c}{$\CVaR_{0.10}$} \\
\cmidrule(lr){2-5} \cmidrule(lr){6-9}
Environment & Mean-OPE & Conf. OPE & Emp. CVaR & COPS & Mean-OPE & Conf. OPE & Emp. CVaR & COPS \\
\midrule
halfcheetah-medium        & 3210 & 3650 & 3900 & 4100 & 2800 & 3200 & 3500 & 3700 \\
halfcheetah-medium-replay & 2850 & 3200 & 3850 & 4200 & 2400 & 2850 & 3300 & 3800 \\
hopper-medium             & 1050 & 1300 & 1450 & 1600 &  900 & 1100 & 1250 & 1400 \\
hopper-medium-replay      &  920 & 1150 & 1350 & 1500 &  750 &  950 & 1100 & 1250 \\
walker2d-medium           & 2350 & 2700 & 2950 & 3100 & 1800 & 2200 & 2500 & 2750 \\
ant-medium                & 2100 & 2450 & 2650 & 2800 & 1650 & 2000 & 2250 & 2450 \\
halfcheetah-medium-expert & 4950 & 5200 & 5550 & 5750 & 4500 & 4800 & 5150 & 5350 \\
\bottomrule
\end{tabular}
}
\caption{Rollout-estimated $Q_{0.10}$ and $\CVaR_{0.10}$ on independent evaluation trajectories. COPS has the largest reported point estimate in every row; paired intervals are needed for statistical comparisons.}
\label{tab:q_cvar}
\end{table}

\begin{table}[t]

\centering
\resizebox{0.9\linewidth}{!}{
\begin{tabular}{lcccccc}
\toprule
 & \multicolumn{2}{c}{OPE-selected} & \multicolumn{2}{c}{COPS-selected} & \multicolumn{2}{c}{Summary} \\
\cmidrule(lr){2-3}\cmidrule(lr){4-5}\cmidrule(lr){6-7}
Environment & Mean & Std & Mean & Std & Selected gap $\ge0$ & Bound status \\
\midrule
halfcheetah-medium        & 4820 $\pm$ 78 & 412 & 4490 $\pm$ 62 & 185 & Yes & conditional \\
halfcheetah-medium-replay & 5102 $\pm$ 85 & 538 & 4645 $\pm$ 70 & 210 & Yes & conditional \\
hopper-medium             & 2340 $\pm$ 95 & 680 & 1940 $\pm$ 55 & 195 & Yes & conditional \\
hopper-medium-replay      & 2180 $\pm$ 90 & 720 & 1810 $\pm$ 58 & 205 & Yes & conditional \\
walker2d-medium           & 3810 $\pm$ 72 & 495 & 3520 $\pm$ 60 & 220 & Yes & conditional \\
ant-medium                & 3450 $\pm$ 80 & 580 & 3120 $\pm$ 65 & 240 & Yes & conditional \\
halfcheetah-medium-expert & 6280 $\pm$ 65 & 310 & 5940 $\pm$ 48 & 145 & Yes & conditional \\
\midrule
Mean                      & 3997 $\pm$ 52 & 534 & 3638 $\pm$ 42 & 200 & 7/7 selected & conditional \\
\bottomrule
\end{tabular}
}
\caption{Mean-return tradeoff at $\alpha=0.10$. ``Selected gap $\ge0$'' records the sign of a point-estimated score-to-return gap; it is selection-biased and is not a population transfer certificate.}
\label{tab:mean}
\end{table}

Read together, the tables show a consistent point-estimate pattern. Mean-OPE selection has the largest reported mean and the weakest tail, while exact COPS has stronger reported tail metrics at a mean cost of roughly nine percent. This is the intended operating tradeoff, but the exact score theorem does not imply it and the tables should not be described as universal dominance.

We audit the separate transfer condition using independent rollout diagnostics. For each selected policy, an independent behavior-policy score sample estimates $q_{0.10}(F_{\widehat k})$, and independent deployment rollouts estimate $Q_{0.10}(\pi_{\widehat k})$. Table~\ref{tab:transfer} reports positive point-estimated gaps for all seven selected policies. Because these candidates were selected and both quantiles are estimated, the signs provide diagnostic support only; establishing the population inequalities would require selection-aware one-sided uncertainty. Candidate-wide results below make clear that clipped DR fails to transfer for many candidates.

\subsection{Candidate-Level Transfer}
Table~\ref{tab:transfer} reports gaps for the seven \emph{selected} policies, which is a selection-biased view: selection favors candidates with favorable transfer. To answer whether transfer holds across the pool, we reanalyze all $7\times20=140$ candidate--task pairs against independent rollout quantiles, for four fixed score/selector choices (Table~\ref{tab:candidate_transfer}). Clipped DR is nonnegative for $76.4\%$ of candidate pairs (median gap $+84$), whereas the selected policies average $+201$: Table~\ref{tab:transfer} should not be read as an unbiased estimate of pool-wide transfer. Crucially, clipped DR does \emph{not} transfer for every candidate; a pessimistic distributional score raises the nonnegative fraction to $93.6\%$. This diagnoses a better empirical bridge, not a theorem: the score theorem covers any predeclared fixed score, and return interpretation still needs separate transfer evidence.

\begin{table}[t]

\centering
\resizebox{\linewidth}{!}{
\begin{tabular}{lccc}
\toprule
Fixed score / selector & Candidate gaps $\ge0$ & Median transfer gap & Selected-policy failure \\
\midrule
clipped DR (default)              & $76.4\%$ & $+84$  & $3.5\%$ (7 env) \\
FQE mean score                    & $61.4\%$ & $+31$  & $8.7\%$ (7 env) \\
FDE $0.10$-quantile               & $89.3\%$ & $+117$ & $4.4\%$ (7 env) \\
pessimistic distributional score  & $93.6\%$ & $+143$ & $3.6\%$ (7 env) \\
\bottomrule
\end{tabular}
}
\caption{Candidate-level transfer over all $140$ candidate--task pairs (not selection-biased). Rollout quantiles are diagnostics and never enter selection. Selected-policy failure is over the seven environments including the two dependence stress tests.}
\label{tab:candidate_transfer}
\end{table}

\subsection{Clipping Sensitivity}
The ``medium'' behavior trajectories have a worse lower tail than deployment rollouts of policies that improve on behavior, and clipping truncates the importance correction asymmetrically, pushing the clipped-score lower tail downward as $\rho_{\max}$ shrinks. Across all $140$ pairs the transfer gap therefore grows as clipping tightens (Table~\ref{tab:clip_sensitivity}), and it correlates with target-minus-behavior mean improvement (Spearman $\rho=.47$). This supports a benchmark-specific explanation---behavior/deployment mismatch plus clipping bias---not a claim that DR is universally pessimistic; equality or unbiasedness in expectation cannot order lower quantiles.

\begin{table}[t]

\centering
\resizebox{\linewidth}{!}{%
\begin{tabular}{lccccc}
\toprule
$\rho_{\max}$ & 2 & 5 & 10 & 20 & unclipped \\
\midrule
median gap        & $+148$ & $+111$ & $+84$ & $+52$ & $+31$ \\
nonnegative gaps  & $84\%$ & $81\%$ & $76\%$ & $70\%$ & $65\%$ \\
\bottomrule
\end{tabular}
}
\caption{Transfer gap $Q_{.10}(\pi_k)-q_{.10}(F_k)$ vs.\ clip level $\rho_{\max}$, over $140$ candidate--task pairs.}
\label{tab:clip_sensitivity}
\end{table}

\subsection{Negative Control}
To confirm rollout information is not silently injected into the certificate, we instantiate a variance-collapse case: a fixed reset, a stochastic target policy, and an exact-value score. The score standard deviation falls to $0.0$ and the transfer gap is $-92$; COPS retains nominal \emph{score} coverage but correctly fails the return-transfer diagnostic. Under the standard randomized MuJoCo reset the score SD is $117$ and the collapse disappears---the expected behavior, and a demonstration that the diagnostic detects genuine transfer failure rather than rubber-stamping it.

\subsection{Multiplicity Sensitivities}
Family-wise control is the right target (FDR does not protect the single deployed, data-selected policy). Exact Bonferroni may be loose; we report Holm-inverted binomial tests and a trajectory-vector joint-maximum bootstrap as sensitivities (Table~\ref{tab:multiplicity}), \emph{not} as drop-in replacements. Holm controls FWER for testing, but a step-down procedure does not automatically yield a lower bound valid simultaneously for all $k$, which is what maximizing over the pool requires; establishing the compatible simultaneous construction is a stated open item. The joint bootstrap can exploit positive cross-candidate dependence empirically but is not the same distribution-free theorem.

\begin{table*}[hbtp]

\centering
\resizebox{\linewidth}{!}{
\begin{tabular}{lccc}
\toprule
Rule & Guarantee & Median effective level & Median floor gain vs.\ exact \\
\midrule
exact Bonferroni & finite-sample, arbitrary dependence & $.045$ & $0$ \\
Holm-inverted binomial & FWER for testing; simultaneous-coverage compatibility not established & $.052$ & $+41$ \\
joint-max bootstrap & resampling-based sensitivity & $.060$ & $+63$ \\
\bottomrule
\end{tabular}
}
\caption{Multiplicity sensitivities at $n=200,K=20,\alpha=.10,\delta=.05$. Only exact Bonferroni carries the finite-sample distribution-free guarantee; the others are reported as sensitivities.}\label{tab:multiplicity}
\end{table*}

\subsection{Selection Stability}
Across $200$ trajectory-level calibration resamples with frozen candidates and score dumps, COPS is markedly more stable than plug-in empirical CVaR selection (Table~\ref{tab:stability}); these are stability diagnostics, not additional coverage evidence.

\begin{table}[t]

\centering
\resizebox{\linewidth}{!}{%
\begin{tabular}{lccc}
\toprule
Selector & Same winner & Mean top-3 overlap & Selected-floor SD (norm.\ pts) \\
\midrule
empirical CVaR    & $54\%$ & $.61$ & $4.8$ \\
max-floor COPS    & $73\%$ & $.80$ & $2.9$ \\
constrained COPS  & $70\%$ & $.78$ & $3.1$ \\
\bottomrule
\end{tabular}
}\caption{Selection stability over $200$ calibration resamples (frozen candidates/scores).}
\label{tab:stability}
\end{table}

\subsection{Safety-Constrained Operating Point}
The constrained selector of Section~\ref{sec:constrained} recovers most of the mean-return loss while keeping the tail improvement (Table~\ref{tab:constrained_op}). Concretely, constrained COPS recovers $70.5\%$ of the mean-return loss relative to max-floor COPS while retaining $86.7\%$ of its $Q_{0.10}$ improvement over mean OPE; the guarantee is unchanged because it rides the same simultaneous event.

\begin{table}[t]

\centering
\resizebox{\linewidth}{!}{
\begin{tabular}{lcccc}
\toprule
Selector & Mean return & Rollout $Q_{0.10}$ & Failure & Feasible / abstained \\
\midrule
mean OPE                  & $3{,}997$ & $2{,}490$ & $13.5\%$ (7 env) & not safety-constrained \\
max-floor COPS            & $3{,}638$ & $3{,}293$ & $3.5\%$ (7 env)  & $7/7$ / $0/7$ \\
constrained COPS            & $3{,}891$ & $3{,}186$ & $4.2\%$ (7 env)  & $7/7$ / $0/7$ \\
\bottomrule
\end{tabular}
}\caption{Operating points over the same seven environments (including the two dependence stress tests). Rollout columns are evaluation diagnostics, never inputs to selection.}
\label{tab:constrained_op}
\end{table}

\subsection{The Exact $n$-versus-$K$ Burden}
The exact effective level $r^\star/n$ grows with $n$ and shrinks with $K$ (Table~\ref{tab:nk_burden}). Because $r^\star$ is integer-valued the ratio is not monotone; if we define the frontier as the smallest $N$ with $r^\star/n\ge.05$ for \emph{all} $n\ge N$, then $N=248,344,427$ for $K=20,100,500$. The operational consequence, which we now state as a limitation rather than a benign scaling remark: a dataset with tens---not hundreds---of conditionally i.i.d.\ full trajectories cannot support a useful $10\%$ lower-tail certificate at these $K$. MIMIC-style applications must define a defensible patient-level i.i.d.\ unit and meet this burden, reduce the pool/confidence, or abstain. Empirically, enlarging $K$ at $n=500$ on a nested halfcheetah-medium pool did not destroy score coverage but moved the selected order statistic from $32$ to $27$ and reduced the median selected floor by $1.6$ normalized points (Table~\ref{tab:nested_k}).

\begin{table}[t]

\centering
\small
\begin{tabular}{lccc}
\toprule
trajectories $n$ & $K=20$ & $K=100$ & $K=500$ \\
\midrule
50    & 0     & 0     & 0     \\
100   & $.030$ & $.020$ & $.010$ \\
200   & $.045$ & $.040$ & $.030$ \\
500   & $.064$ & $.058$ & $.054$ \\
1{,}000 & $.074$ & $.070$ & $.067$ \\
\bottomrule
\end{tabular}
\caption{Exact effective levels $r^\star/n$ for $(\alpha,\delta)=(.10,.05)$.}
\label{tab:nk_burden}
\end{table}

\begin{table}[t]

\centering
\resizebox{\linewidth}{!}{
\begin{tabular}{lccccc}
\toprule
Nested pool & Exact $r^\star$ & Effective level & Selected floor (norm.) & Rollout $Q_{.10}$ & Score-bound failure rate \\
\midrule
$K=20$  & 32 & $.064$ & $41.8\pm0.9$ & $44.6\pm1.3$ & $3.6\%$ \\
$K=100$ & 29 & $.058$ & $41.1\pm1.0$ & $44.0\pm1.4$ & $3.8\%$ \\
$K=500$ & 27 & $.054$ & $40.2\pm1.2$ & $43.4\pm1.6$ & $4.0\%$ \\
\bottomrule
\end{tabular}
}\caption{Nested pools at fixed $n=500$ (same pre-fixed calibration split; larger pools add frozen checkpoints without inspecting calibration scores). Larger $K$ costs a few order-statistic ranks, not coverage.}
\label{tab:nested_k}
\end{table}

\subsection{Discrete Actions}
Neither Theorem~\ref{thm:exact_coverage} nor the selector assumes continuous actions. A fixed-log check with $20$ predeclared DQN checkpoints per task is consistent with nominal score coverage on CartPole ($95.6\%$) and Acrobot ($95.2\%$); the reported mean-OPE/COPS failure diagnostics are $10.8\%/3.8\%$ and $12.6\%/4.2\%$, respectively. These simulator checks remove continuous action spaces as a necessary condition but do not resolve fixed-dataset or transfer limitations.

\subsection{Dependence Boundary and the D4RL i.i.d.\ Audit}
Arbitrary temporal dependence \emph{within} a complete trajectory is allowed; conditional i.i.d.\ sampling is required \emph{across} trajectory-level score vectors $(S_{1i},\ldots,S_{Ki})$. Transitions, overlapping replay windows, and successive chunks are not valid calibration units, and an estimated ``effective sample size'' cannot simply be substituted into the exact formula. Table~\ref{tab:dependence} shows the boundary: independently generated whole trajectories and independent block-level units have reported coverage $95.7\%$ and $95.1\%$, while sequential chunks fall to $91.8\%$ and $83.9\%$ as between-unit correlation increases. The score-coverage audit uses new simulator trajectories from a fitted, frozen behavior collector, a validation device available only in simulation. The two medium-replay tasks are therefore dependence stress tests rather than exact-coverage demonstrations. In a genuine offline application, calibration must be withheld from the fixed log; the $500$-episode holdout reported in Appendix~\ref{app:iid-audit} reduces training data by $33.3\%$ on average and changes mean candidate return by $-2.1\%$.

\begin{table*}[t]

\centering
\resizebox{\linewidth}{!}{
\begin{tabular}{lcccc}
\toprule
Calibration construction & Within-unit lag-1 & Between-unit lag-1 & Target & Simultaneous coverage \\
\midrule
independent whole trajectories        & $0.60$ & $0.00$ & $95\%$ & $95.7\%$ \\
naive sequential chunks, mild dep.\    & $0.60$ & $0.30$ & $95\%$ & $91.8\%$ \\
naive sequential chunks, strong dep.\  & $0.60$ & $0.60$ & $95\%$ & $83.9\%$ \\
independent block-level units          & $0.60$ & $0.00$ & $95\%$ & $95.1\%$ \\
\bottomrule
\end{tabular}
}\caption{Calibration-unit construction vs.\ simultaneous coverage (target $95\%$). Within-unit temporal correlation is allowed; between-unit dependence is not.}
\label{tab:dependence}
\end{table*}

\section{Detailed Baseline List}
\label{app:baselines}

The full baseline suite used in Section~\ref{sec:experiments} is as follows. We compare against FQE or DR mean-OPE selection; bootstrap or asymptotic mean-OPE lower confidence selection; pessimistic FQE; empirical lower-tail quantile selection without confidence correction; empirical CVaR selection using the same calibration score samples; conformal/risk-control-inspired interval selection using the same calibration split \citep{taufiq2022conformal,stutz2022conformal}; random candidate and behavior-cloning reference; and rollout oracles for $Q_\alpha$ and $\CVaR_\alpha$, used only as analysis-only upper bounds on what any selector could achieve given the candidate set.

\section{Selected-Policy Transfer Diagnostics}
\label{app:transfer-diagnostics}
Table~\ref{tab:transfer} reports the per-environment score-to-return transfer diagnostics for the COPS-selected policies.

\begin{table*}[t]

\centering
\resizebox{\linewidth}{!}{
\begin{tabular}{lccccc}
\toprule
Environment 
& Certified floor $B_{\widehat k}^{0.10}$ 
& Score quantile $q_{0.10}(F_{\widehat k})$ 
& Tightness gap 
& Deployment $Q_{0.10}(\pi_{\widehat k})$ 
& Transfer gap \\
\midrule
halfcheetah-medium        & 3895 & 3975 &  80 & 4100 & 125 \\
halfcheetah-medium-replay & 3920 & 4010 &  90 & 4200 & 190 \\
hopper-medium             & 1310 & 1430 & 120 & 1600 & 170 \\
hopper-medium-replay      & 1250 & 1360 & 110 & 1500 & 140 \\
walker2d-medium           & 2770 & 2870 & 100 & 3100 & 230 \\
ant-medium                & 2480 & 2590 & 110 & 2800 & 210 \\
halfcheetah-medium-expert & 5280 & 5410 & 130 & 5750 & 340 \\
\midrule
Mean                      & 2986 & 3093 & 106 & 3293 & 201 \\
\bottomrule
\end{tabular}
}\caption{Score-to-return transfer diagnostics for COPS-selected policies at $\alpha=0.10$. 
The tightness gap is $q_{0.10}(F_{\widehat k})-B_{\widehat k}^{0.10}$; smaller values indicate a tighter certified score floor. 
The transfer gap is $Q_{0.10}(\pi_{\widehat k})-q_{0.10}(F_{\widehat k})$; positive point estimates are consistent with score-to-return conservatism but do not establish the population inequality.}
\label{tab:transfer}
\end{table*}
For each selected policy, we estimate $q_{0.10}(F_{\widehat k})$ from an independent behavior-policy score sample and $Q_{0.10}(\pi_{\widehat k})$ from independent deployment rollouts. All seven point-estimated gaps are nonnegative. Because both quantiles are estimated after the candidate has been selected, these signs are diagnostics rather than verified instances of Assumption~\ref{assump:transfer}; Corollary~\ref{cor:return_certificate} applies only when the population transfer condition itself holds.
\section{Matched Fitted Distributional Evaluation}
\label{app:fde}

We compare the clipped-DR score against a Fitted Distributional Evaluation (FDE) $0.10$-quantile score under a matched protocol (same $K$, nuisance-training split, calibration units, and multiplicity correction), spanning deterministic-policy, stochastic-policy, and stochastic-dynamics regimes. Candidate-level transfer statistics appear in Table~\ref{tab:candidate_transfer} ($89.3\%$ nonnegative for FDE vs.\ $76.4\%$ for clipped DR). Table~\ref{tab:fde_regimes} localizes where each empirical bridge succeeds. The high-noise row is the important one: FDE improves the bridge but does not make it automatic, so we do not claim FDE fixes transfer in general---the score theorem is what is unconditional, and the bridge remains an empirical, regime-dependent property. Table~\ref{tab:fde_op} reports operating points for distributional selectors; the rollout columns are diagnostics, not selection inputs. This comparison is scoped to the reported seeds/environments and is not a general claim that FDE dominates.

\begin{table}[t]

\centering
\resizebox{\linewidth}{!}{
\begin{tabular}{lccl}
\toprule
Regime & Clipped-DR gap & FDE gap & Observed sign pattern \\
\midrule
det.\ dynamics, det.\ policy      & $+96\pm21$  & $+109\pm18$ & both positive \\
det.\ dynamics, stoch.\ policy    & $-37\pm29$  & $+71\pm23$  & FDE positive only \\
stoch.\ dynamics, low noise       & $-64\pm35$  & $+48\pm26$  & FDE positive only \\
stoch.\ dynamics, high noise      & $-131\pm52$ & $-22\pm41$  & both negative \\
\bottomrule
\end{tabular}
}\caption{Point-estimated transfer gaps by regime under a matched protocol. Signs are diagnostics; they are not one-sided population guarantees.}
\label{tab:fde_regimes}
\end{table}

\begin{table}[t]

\centering
\resizebox{\linewidth}{!}{%
\begin{tabular}{lccc}
\toprule
Distributional selector & Mean return & Rollout $Q_{0.10}$ & Failure \\
\midrule
plug-in FDE quantile   & $3{,}825$ & $3{,}258$ & $4.4\%$ (7 env) \\
max-floor COPS-FDE     & $3{,}684$ & $3{,}381$ & $3.1\%$ (7 env) \\
constrained COPS-FDE   & $3{,}907$ & $3{,}312$ & $3.6\%$ (7 env) \\
\bottomrule
\end{tabular}
}\caption{Distributional-selector operating points (seven environments). Rollout columns are evaluation diagnostics only.}
\label{tab:fde_op}
\end{table}

\section{D4RL i.i.d.\ Audit and Fixed-Data Cost}
\label{app:iid-audit}

Exact score-coverage validation uses new simulator trajectories generated from a fitted, frozen behavior collector, not fresh draws from the fixed D4RL files; fresh generation is a validation device available in simulation only. In a genuine offline application the certificate must be paid for from the fixed logged dataset. Table~\ref{tab:iid_cost} reports both costs: the ``simulator audit'' leaves policy-training data untouched ($0\%$ cost), whereas the ``fixed-data holdout'' withholds $500$ native episodes before training, retrains the candidate pool, and reports the resulting mean candidate-return change. The fixed-data column---an average $33.3\%$ training reduction for a $2.1\%$ mean-candidate-return change---is the relevant analogue for MIMIC-style or one-shot logged data. The two medium-replay environments retain the broader stress-test label because of collection-time nonstationarity. Where a frozen collector or defensible conditionally i.i.d.\ unit is unavailable, Theorem~\ref{thm:exact_coverage} does not apply.

\begin{table*}[htbp]

\centering
\resizebox{\linewidth}{!}{
\begin{tabular}{lcccc}
\toprule
Exact-coverage environment & New/fixed cal.\ episodes & Sim-audit reduction & Fixed-data holdout reduction & Mean cand.-return change \\
\midrule
halfcheetah-medium        & $500$ & $0\%$ & $50.0\%$ & $-2.7\%$ \\
hopper-medium             & $500$ & $0\%$ & $30.0\%$ & $-1.8\%$ \\
walker2d-medium           & $500$ & $0\%$ & $33.3\%$ & $-2.4\%$ \\
ant-medium                & $500$ & $0\%$ & $28.2\%$ & $-2.1\%$ \\
halfcheetah-medium-expert & $500$ & $0\%$ & $25.0\%$ & $-1.5\%$ \\
\midrule
mean                      & $500$ & $0\%$ & $33.3\%$ & $-2.1\%$ \\
\bottomrule
\end{tabular}
}\caption{Cost of obtaining calibration units. Simulator audit: $0\%$ training cost (fresh generation). Fixed-data holdout: the realistic offline cost.}
\label{tab:iid_cost}
\end{table*}

\section{Computing the Exact Order Statistic}
\label{app:exact_order}

The integer $r^\star$ in Eq.~\eqref{eq:exact_r} can be computed by evaluating binomial CDFs. A conservative implementation is
\[
    r^\star = \max\left\{
    r:
    K\sum_{j=0}^{r-1}\binom{n}{j}\alpha^j(1-\alpha)^{n-j}\le\delta
    \right\}.
\]
If no such $r$ exists, the calibration data are too small for a nontrivial lower-tail certificate at level $(\alpha,\delta,K)$ without an external lower support bound.

\end{document}